%% file: AdaBKB_1251.tex
\documentclass{article}

\usepackage[round]{natbib}

\usepackage{hyperref}    %
\usepackage{fullpage}
\usepackage{graphicx}
\usepackage{amsmath, amsthm, amssymb}
\usepackage{algorithm}
\usepackage[noend]{algpseudocode}
\usepackage{setspace}

\include{MacroBook}

\makeatletter
\def\BState{\State\hskip-\ALG@thistlm} 
\makeatother

\newtheorem{definition}{Definition}
\newtheorem{lemma}{Lemma}
\newtheorem{thm}{Theorem}
\newtheorem{asm}{Assumption}
\newtheorem{prop}{Proposition}
\newtheorem{remark}{Remark}

\title{\bf Ada-BKB: Scalable Gaussian Process Optimization on Continuous Domains by Adaptive Discretization}

\date{}

\author{%
	{\bf Marco Rando}\hfill \hspace{16.5em}{\small\texttt{marco.rando@edu.unige.it}}\\
	{\small \it MaLGa - DIBRIS, University of Genova, Italy \hfill \hspace{12em}}\\ 
	\and
	{\bf Luigi Carratino} \hfill \hspace{12em}\small \texttt{luigi.carratino@dibris.unige.it}\\
	{\small \it MaLGa - DIBRIS, University of Genova, Italy \hfill \hspace{12em}}\\
	\and
	{\bf Silvia Villa} \hfill \hspace{18.7em}{\small\texttt{silvia.villa@unige.it}}\\
	{\small \it MaLGa - DIMA, University of Genova, Italy \hfill \hspace{12em}}\\   
	\and
	{\bf Lorenzo Rosasco} \hfill \hspace{14.5em}{\small\texttt{lorenzo.rosasco@unige.it}}\\
	{\small \it MaLGa - DIBRIS, University of Genova, Italy  \hfill \hspace{12em}}\\
	{\small \it Istituto Italiano di Tecnologia, Genova, Italy \hfill \hspace{12em}}\\
	{\small \it CBMM - MIT, Cambridge, MA, USA \hfill \hspace{12em}}
}

\begin{document}
\maketitle

\begin{abstract}
Gaussian process optimization is a successful class of algorithms(e.g. GP-UCB) to optimize a black-box function through sequential evaluations. However, for functions with continuous domains, Gaussian process optimization has to rely on 
either a fixed discretization of the space, or the solution of a non-convex optimization subproblem at each evaluation. The first approach can negatively affect performance, 
while the second approach requires a heavy computational burden. A third option, only recently theoretically studied, is to adaptively discretize the function domain. Even though this approach avoids 
the extra non-convex optimization costs, the overall computational complexity is still prohibitive. An algorithm such as GP-UCB has a runtime of $O(T^4)$, where $T$ 
is the number of iterations. In this paper, we introduce Ada-BKB (Adaptive Budgeted Kernelized Bandit), a no-regret Gaussian process optimization 
algorithm for functions on continuous domains, that provably runs in $O(T^2 d_\text{eff}^2)$, where $d_\text{eff}$ is the effective dimension of the explored space, 
and which is typically much smaller than $T$. We corroborate our theoretical findings with experiments on synthetic non-convex functions and on the real-world problem of hyper-parameter optimization, confirming the good practical performances of the proposed approach.
\end{abstract}

\section{INTRODUCTION}

The maximization of a function given only finite, possibly noisy, evaluations is a key and common 
problem in applied sciences and engineering. Approaches to this problem range from 
genetic algorithms~\citep{whitley1994genetic} to zero-th order methods~\citep{nesterov2017random}. Here, we 
take the perspective of bandit optimization, where indeed a number of approaches have been proposed and studied: for example Thompson sampling, or the upper confidence bound algorithm (UCB), see~\citep{lattimore2020bandit} and references therein.
Relevant to our study is a whole line of work developing the basic UCB idea, considering in particular kernels (kernel-UCB)~\citep{kung_2014} or Gaussian processes (GP-UCB)~\citep{rasmussen2003gaussian}. 
In the basic UCB algorithm, the function domain is typically assumed to be discrete (or discretized) and 
an upper bound to the function of interest is iteratively computed and maximized. This approach is sound and amenable to a rigorous theoretical analysis in terms of regret bounds. 
Considering Gaussian processes/kernels, it is possible to extend the applicability of UCB while preserving the nice theoretical properties~\citep{kung_2014,rasmussen2003gaussian}. However, this is at the expenses of computational efficiency. Indeed, 
a number of recent works has focused on scaling UCB with kernels/GP by taking advantage of randomized approximations based on random features~\citep{mutny2019efficient} and Nystrom/inducing points methods~\citep{calandriello2020near,calandriello2019gaussian}, or by performing a smart candidate selection strategy~\citep{calandriello2022scaling}. These studied solutions show that improved efficiency can be achieved without degrading the regrets guarantees. The other line of work relevant to our study focuses on how to tackle functions defined on continuous domains. In particular, 
we consider optimistic optimization, introduced in~\citep{munos2011optimistic} and developed in a number of
subsequent works, see~\citep{valko2013stochastic, kleinberg2013bandits, bubeck2011x, wang2014bayesian, shekhar2018gaussian, salgia2020computationally,kleinberg2008multi}. The basic idea is to iteratively build discretizations in a coarse to fine manner. 
This approach, related to Monte Carlo tree search, can be analyzed theoretically to derive rigorous regrets guarantees~\citep{munos2014bandits}.
In this paper we propose and analyze a novel and efficient approach called Ada-BKB, that combines ideas from optimistic optimization and UCB with kernels. 
A first attempt in this direction has been done in~\citep{shekhar2018gaussian,salgia2020computationally}. However, the corresponding computational costs are 
prohibitive since exact (kernel) UCB computations are performed. So, we take advantage of the latest advances on scalable kernel UCB and adapt optimistic optimization techniques to derive a provably accurate and efficient algorithm. 
Our main theoretical contribution is the derivation of sharp regret guarantees, that shows that Ada-BKB is as accurate as an exact UCB with kernels, with much smaller computational costs. 
We provided an efficient implementation of Ada-BKB which uses techniques such as pruning and early stopping.
We investigate empirically its performance both in numerical simulations and in a hyper-parameter tuning task. 
The obtained results confirm that Ada-BKB is a scalable and accurate algorithm for efficient bandit optimization on continuous domains. The rest of the paper is organized as follows. In Section~\ref{setting}, we describe the problem setting and in Section~\ref{algo}, we describe the algorithm we propose. 
In Section~\ref{theory} and~\ref{sec:experiments} we present our empirical and theoretical results. In Section~\ref{sec:conclusion} we discuss some final remarks. 
\section{PROBLEM SETUP}\label{setting}

Let $(X,d)$ be a compact metric space, for example $X=[0,1]^p \subseteq \mathbb{R}^p$. Let $f:X \rightarrow \mathbb{R}$ be a continuous function and consider the problem of finding
\begin{equation*}
	x^{*} \in \argmax\limits_{x \in X} f(x).    
\end{equation*}
We consider a setting where only noisy function evaluations $y_t = f(x_t) + \epsilon_t$ are accessible. 
Here, $\epsilon_t$ is $\xi$-sub Gaussian noise.
This problem is relevant in black-box or zero-th order optimization~\citep{nesterov}, as well as in muti-armed bandits~\citep{lattimore2020bandit}. 
In this latter context, the function $f$ is also called the \textit{reward function} and $X$ the \textit{arms set}. 
Given $T \in \N$, the goal is to derive a sequence $x_1, \cdots, x_T\in X$, with small cumulative regret,
\begin{equation*}
	R_T = \sum_{t=1}^T (f(x^{*}) - f(x_t)).
\end{equation*}
This can be contrasted to considering the simple regret $S_T = f(x^*) - f(x_T)$ as typically done in optimization. The regret considers the errors accumulated by the whole sequence rather than just the last iteration. 
The sequence $(x_t)_t$ is computed iteratively. At each iteration $t$, an element $x_t \in X$ is selected and a corresponding noisy function value $y_t$ made available. 
The selection strategy, also called a policy, is based on all the function values 
obtained in previous iterations.
In the following we assume $f$ to belong to a 
reproducing kernel Hilbert space (RKHS). 
The latter is a Hilbert space of $(\hh, \scal{\cdot}{\cdot}, \|\cdot\|)$ of functions from $X$ to $\R$, with associated a function $k:X\times X\to \R$, called reproducing kernel or kernel,  such that for all $x\in X$ and $f'\in \hh$, 
\begin{equation*}
	k(x, \cdot)\in \hh, \qquad\text{and}\qquad f'(x)=\scal{f'}{k(x,\cdot)}.  
\end{equation*}
We assume that $k(x,x) \leq \kappa^2$ for all $x\in \X$ and $\kappa \geq 1$. 
We let $d_k : \X \times \X \rightarrow [0,\infty)$ be the distance in the RKHS $\hh$ defined as 
$d_k(x,x')= \nor{k(x, \cdot)-k(x',\cdot)}= \sqrt{k(x,x)+k(x',x')-2k(x,x')} $ with $x,x'\in X$. 
Further, we consider kernels for which the following assumptions hold.

\begin{asm}\label{asm:smoothd_k}
	There exists a non-decresing function $g:[0,\infty)\to [0,\infty)$ such that $g(0) = 0$ and for all $x,x'\in X$ 
	\begin{equation}\label{eq:smoothd_k}
		d_k(x,x') \leq g( d(x,x')). 
	\end{equation}
\end{asm}

\begin{asm}\label{asm:bound_g}
	Let $g$ be the non-decreasing function indicated in Assumption~\ref{asm:smoothd_k}. There exist $\delta_k > 0$, $\alpha \in (0,1]$, and $C^\prime_k, C_k > 0$ such that
	\begin{equation}
		(\forall r \leq \delta_k) \qquad  C_k r^{\alpha} \leq g(r) \leq C^\prime_k r^{\alpha}
	\end{equation}
\end{asm}
It is easy to see that, the above condition is satisfied, for example, for the Gaussian kernel $k(x_1, x_2) = e^{-\frac{\nor{x_1 - x_2}^2}{l}}$ with  $\alpha = 1$ and suitable constants $\delta_k,C_k,C_k'$, for $g(r) = \sqrt{\frac{2}{l}} r$. 

\section{ALGORITHM}\label{algo}
The new algorithm we propose combines ideas from AdaGP-UCB~\citep{shekhar2018gaussian} and BKB~\citep{calandriello2019gaussian} 
(a scalable implementation of GP-UCB/KernelUCB~\citep{srinivas2009gaussian,kernelUCB}). 
We begin recalling the ideas behind GP-UCB and BKB.

\paragraph{From kernel bandits to budgeted kernel bandits.} 
The basic idea in GP-UCB/KernelUCB is to derive
an upper estimate $f_t$ of $f$ at each step,
and then select the new point $x_{t+1}$ maximizing such an estimate.
The upper estimate is defined using a reproducing kernel $k : \X \times\X \rightarrow \R$. 
Let $(x_1,y_1), \dots, (x_t,y_t)$ be the sequence of evaluations points and noisy evaluation 
values up-to the $t$-th iteration. 
Let $K_t\in \R^{t \times t}$ be the matrix with entries $(K_t)_{ij}=k(x_i,x_j)$, 
for $i,j=1, \dots,t$, denote $k_t(x)= (k(x,x_1), \dots, k(x, x_t))\in\R^t$ and $Y_t=(y_1, \dots, y_t)\in\R^t$.
For $\la>0$, let
\begin{equation}\label{eqn:exact_mu_sig}
	\begin{aligned}
		\mu_t(x) &= k_t(x)^\top (K_t+\la I)^{-1}Y_t \\
		\sigma_t(x)^2 &= k(x,x)-k_t(x)^\top (K_t+\la I)^{-1}k_t(x).
	\end{aligned}   
\end{equation}
For $\beta_t>0$, the upper estimate of $f$, known as upper confidence bound (UCB), is defined as
\begin{equation*}
	f_t(x)= \mu_t(x)+\beta_t \sigma_t(x).  
\end{equation*}
Note that $\la$ and $\beta_t$ are parameters that need to be specified.
The quantities $\mu_t, \sigma_t$ can be seen as a kernel ridge regression estimate and a suitable {\em confidence bound}, respectively. Also, they have a natural Bayesian interpretation in terms of mean and variance of the posterior induced by a Gaussian Process, hence the name GP-UCB~\citep{srinivas2009gaussian}. 
KernelUCB/GP-UCB have favorable regret guarantees \citep{kernelUCB,srinivas2012gaussian}, but 
computational requirements that prevent scaling to large data-sets. 
BKB~\citep{calandriello2019gaussian} tackles this issues considering a \Nystrom{}-based 
approximation~\citep{drineas2005nystrom}. 
Let $X_t = (x_1, \dots, x_t)\in \R^{t \times p}$ be the collection of evaluation points up-to 
iteration $t$ and $S_t \subseteq X_t$ a subset of cardinality $m \leq t$. 
Let $K_{S_t}\in \R^{m\times m}$ such that $(K_{S_t})_{ij}=k(x_i,x_j)$ with $x_i,x_j\in S_t$,
and $k_{S_t}(x)\in \R^m$ such that $(k_{S_t}(x))_i = k(x,x_i)$ with $x_i \in S_t$.
Let $\widetilde{k} : \X \times \X \rightarrow \R$ be the approximate \Nystrom{} kernel defined as
\begin{equation}\label{eq:approx_kernel}
	\widetilde{k}(x, x') = k_{S_t}(x)^\top K_{S_t}^\dagger k_{S_t}(x').
\end{equation}
Let $\widetilde {K}_{S_t}\in \R^{t\times t}$ such that $(\widetilde{K}_{S_t})_{ij}=\widetilde{k}(x_i,x_j)$ with $x_i,x_j\in X_t$,
and $\widetilde{k}_{S_t}(x)\in \R^t$ such that $(\widetilde{k}_{S_t}(x))_i = \widetilde{k}(x,x_i)$ with $x_i \in X_t$.

For $\la > 0$, let
\begin{equation}\label{eqn:mu_sig}
	\begin{aligned}
		\ti{\mu}_t(x_i) &= \widetilde{k}_{S_t}(x_i)^\top (\widetilde{K}_{S_t} + \la I)^{-1}Y_t \\ 
		\ti{\sigma}^2_t(x_i) &= \frac{1}{\lambda}(k(x_i,x_i) - \widetilde{k}_{S_t}(x_i)^\top (\widetilde{K}_{S_t} + \la I)^{-1} \widetilde{k}_{S_t}(x_i) )   
	\end{aligned}    
\end{equation}
and, for $\beta_t > 0$
\begin{equation}\label{bkb}
	\ti{f}_t(x) = \ti \mu_t(x)+\beta_t \ti \sigma_t(x).
\end{equation}

BKB uses the above approximate estimate and select at each iterations the points in 
$S_t$ proportionally to their variance at the previous 
iterate $\ti{\sigma}_{t-1}^2(x_i)$~\citep{calandriello2019gaussian}.
This sampling strategy guarantees that, for the proper values of $\beta_t$, 
$|S_t| \leq O(d_\text{eff}(t))$ 
where $\ti{X}_t$ is the set of explored points until function evaluation $t$ and $d_\text{eff}$ is the effective dimension, a quantity typically much lower than $t$ and defined as
\begin{equation}\label{eqn:deff}
	d_\text{eff}(t) = \sum\limits_{t = 1}^{T}\sigma_t^2(x_t).
\end{equation}
where with $x_t$ is the point evaluated at time $t$. 
To maximize the upper estimate ($f_t$ for KernelUCB/GP-UCB and $\ti{f}_t$ for BKB) these algorithms 
rely on the assumption that the arms set $\X$ is discrete.
In practice, when $\X$ is continuous, a fixed discretization is considered. 
In the next section we discuss how the latter can be computed adaptively and
introduce some necessary concepts and assumptions.

\paragraph{Partition Trees.} 
Key for adaptive discretization is a family of partitions called partition trees. 
Following~\citep{shekhar2018gaussian}, the notion of partition tree for metric spaces 
is formalized by the following definition. 

\begin{definition}
	\label{def:wellbehaved}
	Let $(X_h)_{h \in \N}$ be families of subsets of $X$, with $X_0=X$. 
	For each $h\in \N$ (called depth), the family of 
	subsets $X_h$ has cardinality $N^h$ with $N\in\N$. 
	The elements of $X_h$ are denoted by $X_{h,i}$ and called cells.
	Each cell $X_{h,i}$ is identified by the point $x_{h,i}\in X_{h,i}$ (called centroid)
	such that
	\begin{equation*}
		X_{h,i} = \{x \in X : d(x,x_{h,i}) \leq d(x,x_{h,j}) \quad \forall j \neq i\}.
	\end{equation*}
	Further, for all $h\in \N$ and $i= 1,\dots, N^h$, 
	\begin{equation*}
		X_{h,i} = \cup_{j = N(i - 1) + 1}^{Ni}X_{h+1,j}.
	\end{equation*}
	The cells $(X_{h+1,j})_j$ are called children of $X_{h,i}$, and $X_{h,i} $ is called parent of $(X_{h+1,j})_j$.      
	
\end{definition}

Note that each cell $X_{h,i}$ identifies a node in the tree denoted by the index $(h,i)$.
To describe the above parent/children relationship we define the following function on indexes.
Let $(0,1)$ be the index of the root cell $X_{0,1} = X$,
we denote with $p$ that function that given the index of a cell $(h+1,j)$ returns the index of its parent $(h,i)$,
and with $c$ that function
that given the index of a cell $(h,i)$ returns the indexes of its 
children $\{(h+1, N(i - 1) + 1), \dots, (h+1, N i)\}$. 
In the following we refers to $p$ and $c$ as parent function and children function.

\paragraph{Partition growth and maximum local reward variation.}

We make the following assumption which formalizes the idea that the cell size decreases with depth.
\begin{asm}\label{asm:cell_radius}
	Let $B(x, r, d)$ be a $d$-ball with radius $r$ and centered in $x$, we assume that there exist $\rho \in (0,1)$ and $0 < v_2 \leq 1 \leq v_1$ such that for $h\ge 0$ and all $i=1, \dots, N^h$
	\begin{equation*}
		B(x_{h,i}, v_2\rho^h, d) \subset X_{h,i} \subset B(x_{h,i}, v_1\rho^h, d)
	\end{equation*}
\end{asm}
Knowing that $f \in \hh$, from the above assumption and Assumption~\ref{asm:smoothd_k} 
we can derive the following upper bound on the maximum variation of 
$f$ in the cells $(X_{h,i})_i$ at each depth $h$.
\begin{lemma}\label{lm:val_vh}
	Under Assumptions~\ref{asm:smoothd_k} and \ref{asm:cell_radius}, let $f\in\hh$ and let $F=\|f\|$. Then, 
	for all $h\geq 0$ and for all $1 \leq i \leq N^h$,
	\begin{equation}\label{eqn:vh}
		\sup\limits_{x,x' \in X_{h,i}} |f(x) - f(x')| \leq V_h
	\end{equation}
	with $V_h = Fg(v_1\rho^h)$
\end{lemma}
We provide the proof in Appendix~\ref{proof:vh}
\subsection{Ada-BKB} 
We now present the new algorithm called Adaptive-BKB (Ada-BKB).
Given a partition tree and a function evaluation budget $T$, 
the basic idea is to explore the set of arms in a coarse to fine fashion, 
considering a variation of BKB on the cells' centroids of the partition tree. 
The algorithm is given in Algorithm~\ref{alg:1} and we next describe its various steps.

\paragraph{Preliminaries: index function and leaf set.}
Recalling the definition of the parent function $p$, given $x_{h,i} \in X_{h,i}$
we let $x_{p(h,i)}$ be the centroid of the parent cell. Then, we define the so called index function as
\begin{equation}\label{eqn:index}
	I_t(x_{h,i})= \min(\ti{f}_t(x_{h,i}),\ti{f}_t(x_{p(h,i)}) + V_{h-1}) + V_h   
\end{equation}
with $\ti{f}_t$ as in~\eqref{bkb}. 
In other terms, we compute an high probability upper bound of $f$ on $x_{h,i}$ and, 
adding $V_h$, we get an high probability upper bound over the maximum values of $f$ in the cell $X_{h,i}$.

Ada-BKB proceeds iteratively.
The algorithm maintains two counters, 
$\tau$ which counts the total number of function evaluations and refinements (see below), 
and $t$ which keeps track of the number of function evaluations performed.
A set of cells' centroids $L_\tau$ (called the leaf set) is updated at each iteration $\tau \geq 0$.
We next describe how the leaf set is used and populated recursively.
\paragraph{First evaluation-update steps.}
The leaf set initially contains only the centroid of root cell, that is 
\begin{equation*}
	L_0=\{x_{0,1}\}.  
\end{equation*}
The function value is queried at $x_{0,1}$ to obtain $y_1= f(x_{0,1})+\eps_1$ and 
the first estimates $\ti{\mu}_1, \ti{\sigma}_1$ are computed. 
Then, given a suitable parameter $\beta_t$, the condition, 
\begin{equation*}
{\beta_t} \ti{\sigma}_{1}(x_{0,1}) \leq V_0,  
\end{equation*}
is checked. 
Initially the term $\ti{\sigma}_{1}(x_{0,1})$ is typically large and the condition is violated. 
In this case, another function value
\begin{equation*}
y_2= f(x_{0,1})+\eps_2 
\end{equation*} 
is queried to derive new estimates $\ti{\mu}_2, \ti{\sigma}_2$ using all available data. Then, the  
condition ${\beta_t} \ti{\sigma}_{2}(x_{0,1}) \leq V_0$ is checked again. 
If violated more function values $y_{t}= f(x_{0,1})+\eps_{t}$ are queried, and estimates
$\ti{\mu}_{t}, \ti{\sigma}_{t}$ computed, until the condition 
${\beta_t} \ti{\sigma}_{t}(x_{0,1}) \leq V_0$ is satisfied . 
Both counters are updated i.e. $\tau = t$.
\paragraph{First leaf-set-expansion step.} 
During all the above iterations the leaf set is unchanged, 
so that $L_{\tau}=L_0$. When the condition ${\beta_t} \ti{\sigma}_{t}(x_{0,1}) \leq V_0$ 
is satisfied, then the leaf set is expanded according to the following rule
\begin{equation*}
L_{\tau+1} = (L_\tau \setminus \{x_{0,1}\}) \cup \{x_{1,j} |  1 \leq j \leq N\},  
\end{equation*}
and the counter $\tau$ is incremented by $1$.
In words, the cell we just evaluated is taken off the leaf set and its children included.
\begin{figure}[H]
\centering
\vspace{.3in}
\includegraphics[width=0.95\linewidth]{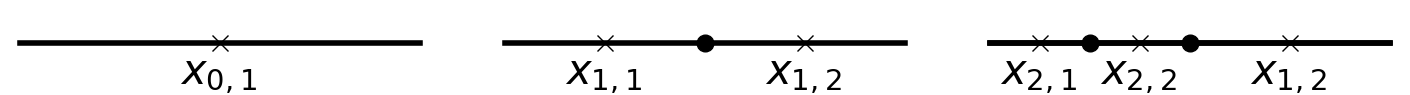}
\vspace{.3in}
\caption{Description of the first and second refinement procedures. 
	The $x_{h,i}$ are the centroids contained in the leaf set while 
	the $\bullet$ represent the centroid removed after the refinement procedure. 
	From left to right, the initial state of the leaf set (containing only the 
	centroid of the root cell), the first refinement and a second refinement with number of 
	children per cell $N = 2$.}
\end{figure}
\paragraph{Further evaluation-update steps.} The estimates 
$\ti{\mu}_{t}, \ti{\sigma}_{t}$ are computed\footnote{Notice that the computation include re-sampling the points in $S_t$
proportionally to $\ti{\sigma}^{2}_{t-1}(x_i)$~\citep{calandriello2019gaussian}} and used to build 
$I_t$ as in~\eqref{eqn:index}.  
Then, the cell $ x_{1,i}$ in the 
leaf set $L_\tau$ maximizing the index function is selected, 
\begin{equation*}
x_{1,i}=\argmax_{x\in L_\tau} I_t(x).  
\end{equation*}
The condition ${\beta_t} \ti{\sigma}_{t}(x_{1,i}) \leq V_1$ is then checked. If violated a value $y_{t+1}= f( x_{1,i})+\eps_{t+1}$ is queried and then the estimates $\ti{\mu}_{t+1}, \ti{\sigma}_{t+1}$and $I_{t+1}$ computed. A new cell is then selected as above
\begin{equation*}
x_{1,i'}=\argmax_{x\in L_{\tau+1}} I_{t+1}(x).  
\end{equation*}
Note that, we might obtain the same cell $i=i'$ or a different cell $i\neq i'$. 
Again the condition ${\beta_t} \ti{\sigma}_{t+1}(x_{1,i'}) \leq V_1$ is checked until satisfied,
and this can entail querying multiple evaluations, possible at more cells. 
\paragraph{Further leaf-set-expansion steps.} 
Note that, throughout the possible function evaluations the leaf set remains unchanged.
Also, while we might evaluate multiple cells, at some point 
the condition ${\beta_t} \ti{\sigma}_{t+1}(x_{1,i'}) \leq V_1$ will be satisfied by a given cell. 
Then, indicating with $c(\cdot)$ the function which given a centroid 
returns the set of children of node represented by the given centroid i.e.
\begin{equation*}
	c(x_{h,i}) = \{x_{h+1,j} | N(i - 1) + 1 \leq j \leq Ni\}
\end{equation*}
the leaf set will be updated as follow 
\begin{equation*}
	L_{\tau+1} = (L_\tau \setminus \{x_{h,i'}\}) \cup \text{c}(x_{h,i'})
\end{equation*}
The cell $x_{h,i'}$ we last evaluated is taken off the leaf set, 
its children $x_{h+1,j}$, $ N(i - 1) + 1 \leq j \leq Ni$ added, 
but note that also all the cells $x_{h,i}$, $i\neq i'$ in the same partition 
as $x_{h,i'}$ are kept in the leaf set. 
Moreover, in order to avoid the (unlikely) scenarios in which the algorithm 
keeps refining indefinitely without evaluating the function, 
a maximum depth threshold $h_{\text{max}}$ is added. 
\begin{figure}[H]
	\centering
	\vspace{.3in}
	\includegraphics[width=0.32\linewidth]{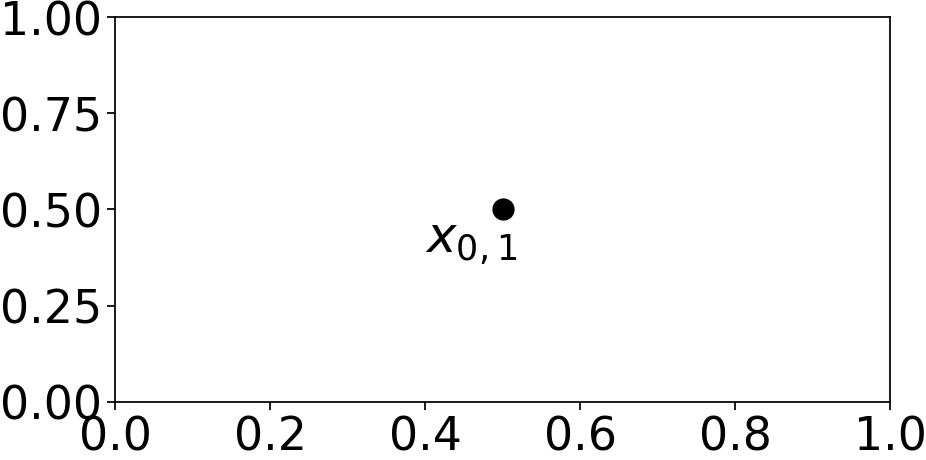}
	\includegraphics[width=0.32\linewidth]{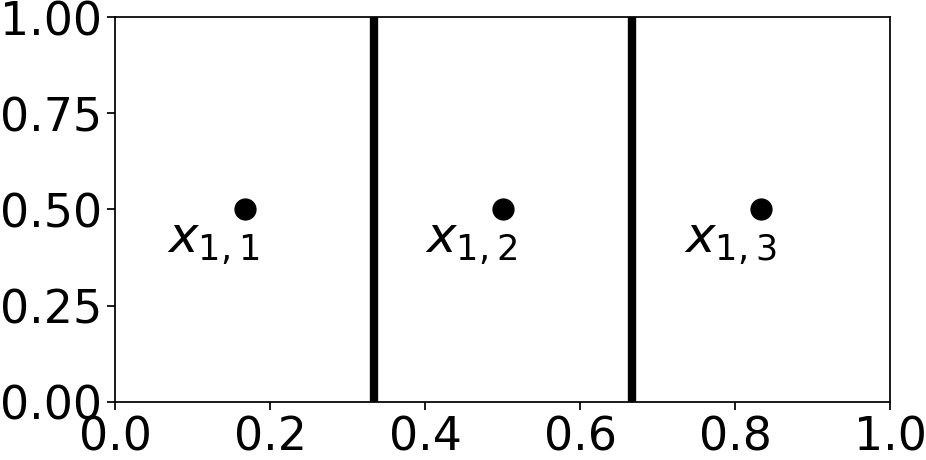}
	\includegraphics[width=0.32\linewidth]{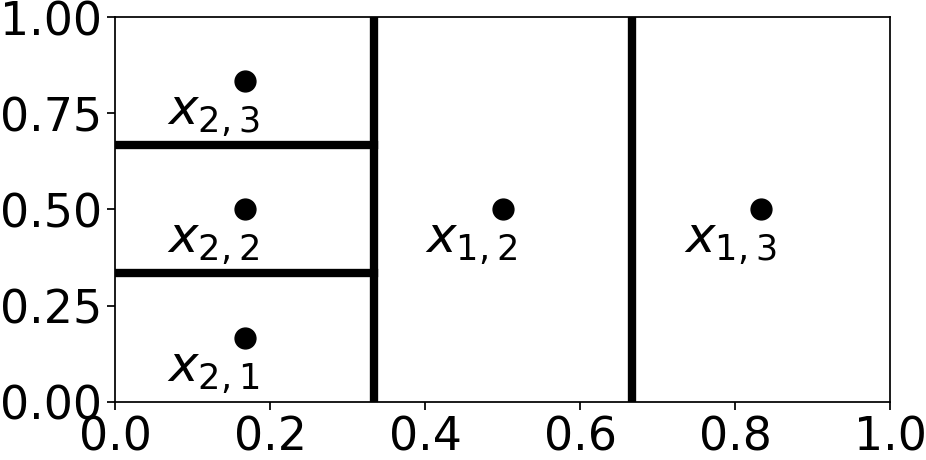}
	\vspace{.3in}
	\caption{Consider $X = [0,1]^2$ and $N = 3$. 
		Here $\bullet$ denotes the centroids. 
		The first picture (from top to bottom), represent the initialization of the algorithm 
		where we have only the root ($X = X_{0,1} = [0,1]^2$); 
		the second picture, represent the first refinement in which we split the 
		root cell in $N = 3$ cells associated to the children 
		($(1,1)$ has $X_{1,1} = [0, 1/3] \times [0, 1]$, $(1,2)$ has 
		$X_{1,2} = [1/3, 2/3] \times [0, 1]$ and $(1,3)$ has 
		$X_{1,3} = [2/3,1] \times [0, 1]$). The third picture, 
		represent the expansion of cell $(1,1)$.}
\end{figure}

\paragraph*{Pruning rule.} 
One of the core differences between Ada-BKB and AdaGP-UCB is the presence of a pruning rule. 
This rule eliminates the cells that in high probability don't contain a global maximizer. 
Let $X_t$ be the set of centroids observed until time $t$, and let the highest lower confidence bound (LCB) be defined as
\begin{equation*}
	l^*_t = \max\limits_{x \in X_t} \tilde{\mu}_t(x) - \tilde{\beta}_t\tilde{\sigma}_t(x)
\end{equation*}
After each iteration, the pruning rule erases every centroid in the leaf set $L_\tau$ that have their upper bound on the maximum over the cell smaller than $l^*_t$.
Formally, we define a function 
\textit{er}$_t: X \rightarrow \{0, 1\}$ 
which, given a centroid $x_{h,i}$, returns $1$ if the centroid needs to be pruned and $0$ otherwise
\begin{equation*}
	\text{er}_t(x_{h,i}) =
	\begin{cases}
		1 & \text{if} \quad \tilde{f}_{t-1}(x_{h,i}) + V_h < l^{*}\\
		0 & \text{otherwise}
	\end{cases}  
\end{equation*}
Thus, the leaf set is updated as 
$L_{\tau + 1} = L_{\tau + 1} \setminus \{ x_{h,i} \in L_{\tau + 1}: \text{ er}(x_{h,i}) > 0 \}$.
Notice that this pruning rule doesn't increase the computational cost since all the information used for the check must be computed previously for different reasons (as the UCB + $V_h$) and the best lower bound can be stored and updated after every evaluation (the informations used for the best lower bound, i.e. $\tilde{\mu}_t$ 
and $\tilde{\beta}_t\tilde{\sigma}_{t-1}$, are already computed for the index function). 
Notice that if an expansion is performed the centroids to check are just the new ones 
(since the model is not updated).
 
Moreover, this pruning rule automatically provide us an early stopping condition, infact, 
if after the pruning procedure the leaf set size is $0$ or $1$ and the only centroid contained in the set is 
$x_{h_\text{max},i}$, we can interrupt the execution and terminate the algorithm 
since every subsequent evaluation will be performed on this centroid. 
In practice, this procedure is very useful because 
it allows to limit the effects of over-expansion of the tree that would make 
the algorithm very time-expensive (see Section~\ref{sec:experiments} and Appendix~\ref{app:other_exp}). 
\begin{algorithm}[H]
	\setstretch{1.1}
	\caption{Ada-BKB}
	\begin{algorithmic}[1]
		\State{\textbf{Input:} $T > 0$, $h_{\text{max}}$, $N$, $\beta_t$}
		\State{Initialize $L_0 = \{x_{0,1}\}, \tau = 0, t = 1$}
		\While{$t \leq T$}
		\State{$x_{h,i} = \argmax\limits_{x_i \in L_\tau} I_t(x_i)$}
		\If{${\beta_t} \ti{\sigma}_{t-1}(x_{h,i}) \leq V_h$ and $h_t < h_{\text{max}}$}
		\State{$L_{\tau+1} = (L_\tau \setminus \{x_{h,i}\}) \cup c(x_{h,i})$}
	\Else
	\State{$y_t = f(x_{h,i}) + \epsilon_t$ (with $\epsilon_t$ noise)}
	\State{compute $\ti{\mu}_{t+1}, \ti{\sigma}_{t+1}, l^*_{t+1}$}
	\State{$L_{\tau + 1} = L_{\tau}$}
	\State{$t = t + 1$}
	\EndIf
	\State{$L_{\tau + 1} = L_{\tau + 1} \setminus \{x_{h,i}: \text{ er}(x_{h,i}) > 0\}$}
	\If{$|L_{\tau + 1}| == 0$ or $L_{\tau + 1} == \{ x_{h_\text{max},i}\}$ }
	\State{break}
	\EndIf
	\State{$\tau = \tau + 1$}
	\EndWhile
\end{algorithmic}\label{alg:1}
\end{algorithm}
Note that, when performing a leaf-set-expansion step, we have yet to specify how to choose N children.
Thus we consider 
the refinement of a cell $X_{h,i}$ is performed by dividing it equally in $N$ parts along its longest side. 
This is a common method which allows to get a partition tree defined as in 
Definition~\ref{def:wellbehaved} satisfying also Assumption~\ref{asm:cell_radius} 
as shown in~\citep{shekhar2018gaussian,bubeck2011x,salgia2020computationally}

\section{MAIN RESULTS}\label{theory}
In this section we present the two main theorems of the paper. 
Theorem~\ref{thm:reg_bounds} shows that the regret bounds for 
Ada-BKB are the same as those of exact GP-UCB, 
while in Theorem~\ref{thm:comp_cost} we prove that the computational cost of 
Ada-BKB is smaller that the one of other adaptive methods. 
Altogether, our results show that, thanks to the use of sketching, Ada-BKB a fast adaptive method achieving state-of-the-art regret bounds. 

\subsection{Regret Analysis}
We next present the first main contribution of the paper on the cumulative regret, for a given function in the considered reproducing kernel Hilbert space. We recall that we have access to noisy function evaluations $y_t = f(x_t) + \epsilon_t$, where $\epsilon_t$ is $\xi$-sub Gaussian.

\begin{thm}[Regret Bounds]\label{thm:reg_bounds} Let $f\in\hh$, and let $F=\|f\|$.
	Let $\delta \in (0, 1)$, $\epsilon \in (0,1)$ and $\bar{\alpha} = \frac{1 + \epsilon}{1 - \epsilon}$. 
	Suppose that Assumptions~\ref{asm:smoothd_k},\ref{asm:bound_g},\ref{asm:cell_radius} are satisfied. 
	Consider Ada-BKB (Alg.~\ref{alg:1}) with $N \geq 1$, $T\geq (v_1/\delta_k)^{2\alpha}$, 
	$h_\text{max} \geq \frac{\log(T)}{2\alpha\log(1/\rho)}$, 
	$\lambda = \xi^2$, 
	$\zeta_t = \bar{\alpha} \log(\kappa^2 t) \Big( \sum\limits_{s = 1}^{t} \ti{\sigma}_t^2(x_s) \Big)$ and $\beta_t$ defined as
	\begin{equation}\label{eqn:beta}
		{\beta}_t = 2\lambda^2 \sqrt{\zeta_t + \log (1/\delta)} + \Big( 1 + \frac{1}{\sqrt{1 - \epsilon}}\Big)\sqrt{\lambda}F.
	\end{equation}
	Then, with probability at least $1 - \delta$, 
	\begin{equation}    
		R_T \leq \mathcal{O}(\sqrt{T}d_\text{eff}(T)\log(T)).
	\end{equation}
	Moreover, if the evaluation model is $y_t = f(x_t) + \eta_t \quad \eta_t \sim \mathcal{N}(0, \sigma^2)$, the cumulative regret can be bounded as:
	\begin{equation}    
		R_T \leq \mathcal{O} \Bigg(\sqrt{Td_{\text{eff}}(T)\log(T)\frac{N^{h_\text{max}} - 1}{N - 1}} \Bigg).
	\end{equation}
	
\end{thm}

The above Theorem shows that the regret bound for Ada-BKB matches exactly the regret bounds of the non-adaptive methods BKB and BBKB~\citep{calandriello2020near}.
The comparison is straightforward, since the bounds for all the methods are expressed in terms of the same quantities.
AdaGP-UCB and the non-adaptive methods GP-UCB~\citep{srinivas2009gaussian}, TS-QFF~\citep{mutny2019efficient} have a regret of $O(\sqrt{T}\gamma_T)$, where 
$\gamma_T$ is the mutual information gain. 
It is shown in~\citep{calandriello2019gaussian} that $\gamma_T$ is of the same order of $d_\text{eff}(T)$, and therefore
the regret bounds for Ada-BKB are better when $\sqrt{\log(T)\frac{N^{h_\text{max}} - 1}{N - 1}} \leq \sqrt{d_\text{eff}(T)}$.
Finally, we recall that GP-ThreDS~\citep{salgia2020computationally} has a regret bound of $O(\sqrt{T \gamma_T}(\log T)^2)$, namely $O(\sqrt{T d_\text{eff}(T)}(\log T)^2)$ 
and thus in this case Ada-BKB can be advantageous if $\sqrt{\log(T)\frac{N^{h_\text{max}} - 1}{N - 1}} \leq (\log T)^2$. 
We extend the discussion in appendix~\ref{app:expanded_discussion}
\subsection{Computational Cost Analysis}
In this section we compute the total computational cost of Ada-BKB, for a specific choice of the family of partition, in the case $X=[0,1]^p$. 
The computational cost of Ada-BKB is due to the following operations: 1) the computation of $\ti{f}_t$, 2) the computation of $I_t(x)$ for all $x \in L_\tau$,
3) the discretization refinement.
We bound each cost separately.

1) The cost of computing $\ti{f}_t$ is the cost of computing $\ti{\mu}_t, \ti{\sigma}_t$ and $\beta_t$. 
The time complexity of computing these quantities over $T$ observations is $\mathcal{O}(Td^2_\text{eff}(T) )$~\citep{calandriello2019gaussian}.

2) Since the evaluation cost of $\ti{\mu}_t, \ti{\sigma}_t$ is bounded by $\mathcal{O}(d_\text{eff}^2(t))$, the worst case cost of evaluating $I_t$ on the leaf set is
$\mathcal{O}(Td^2_\text{eff}(T)N^{h_{\text{max}}})$

3) For $X = [0, 1]^p$ with the euclidean norm, consider the following rule to refine the partition from level $h$ to $h+1$. $X_{0,1}$ is cut 
along one of its sides in $N$ equal parts, obtaining $N$ rectangles. 
Then, each set $X_{h,i}$ in the partition $X_h$ is divided in $N$ parts equally again along the longest side. 
This partition is built using the same refinement procedure used in~\citep{shekhar2018gaussian} which costs $\mathcal{O}(TpNh_{\text{max}})$. 
\begin{thm}[Computational Cost]\label{thm:comp_cost}
	Let $X =[0, 1]^p$ endowed with the euclidean distance. Then, Ada-BKB with the same parameters as in \ref{thm:reg_bounds} has time complexity 
	\begin{equation*}
		\mathcal{O}(Td_\text{eff}^2(T) N^{h_{\text{max}}} + TpN h_{\text{max}})    
	\end{equation*}
\end{thm}
\begin{remark}
	Using the arguments in~\citep{shekhar2018gaussian}, for $N$ odd, the leaf set size is bounded, for every $\tau$, by
	\begin{equation*}
		|L_\tau| \leq T N h_{\text{max}}.
	\end{equation*}
	
	Then, for a fixed $p$ and $N$ the overall computational cost become:
	\begin{equation*}
		\mathcal{O}(T^2 d_\text{eff}^2(T) h_{\text{max}}).    
	\end{equation*}
	
\end{remark}
\paragraph*{Discussion on Computational Cost.}
Ada-BKB has the provably smallest computational complexity of all methods with adaptive discretization which can deal with noisy observation cases:
Ada-GPUCB costs $\mathcal{O}(T^4(N - 1)h_{\text{max}} +TpN h_{\text{max}})$, GP-ThreDS costs $O(T^4)$. 
Note that GP-ThreDS has a computational complexity which is independent from $p$ while 
Ada-BKB and Ada-GPUCB are linear in the dimension. 
Comparing our algorithm with GP-UCB ($\mathcal{O}(T^3A)$ with $A$ size of the discretization of $X$), we note that we get smaller computational cost 
in most cases. Indeed, usually the cardinality of the discretization grows exponentially with the dimension of $X$. 
Analogously, in the same setting, our algorithm is faster than BKB ($\mathcal{O}(TAd^2_\text{eff})$) and TS-QFF($\mathcal{\tilde{O}}(TA2^pd_\text{eff})$)~\citep{mutny2019efficient}. 

\section{EXPERIMENTS}\label{sec:experiments}

In this section, we study the empirical performances of Ada-BKB compared 
with GP-UCB~\citep{srinivas2009gaussian}, BKB~\citep{calandriello2019gaussian} 
and AdaGP-UCB~\citep{shekhar2018gaussian}. 
We refer to Appendix~\ref{app:experiements} for further details and results. 
The hyperparameters of the algorithms are fixed according to theory, or, when not possible, 
by cross-validation, as for the kernel parameters. 
\paragraph*{Function minimization.} 
We consider the minimization of a number of 
well known functions corrupted by Gaussian noise 
with zero mean and standard deviation $0.01$. 
For GP-UCB and BKB, a fixed discretization of the function domain is considered. 
For each experiment we report mean and a $95\%$ confidence interval using $5$ repetitions. 
\begin{figure}[H]
	\centering
	\includegraphics[width=0.24\linewidth]{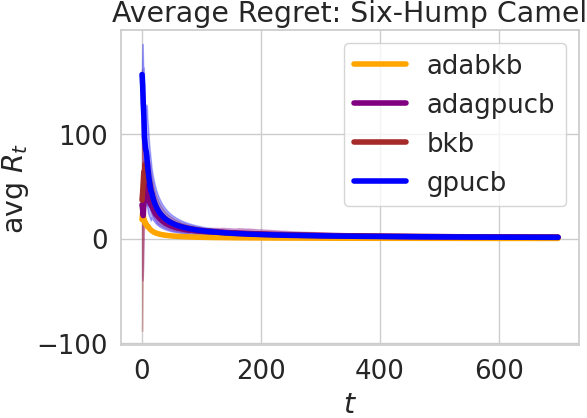}
	\includegraphics[width=0.24\linewidth]{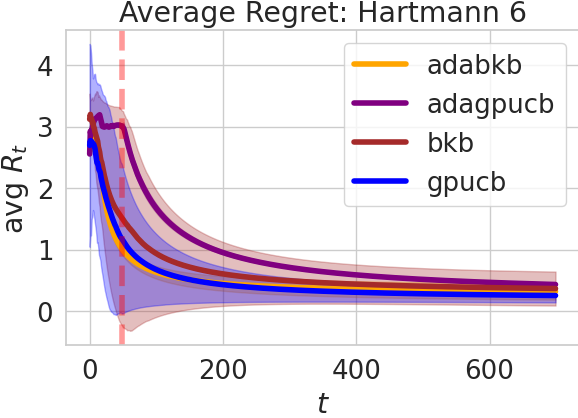}
	\includegraphics[width=0.24\linewidth]{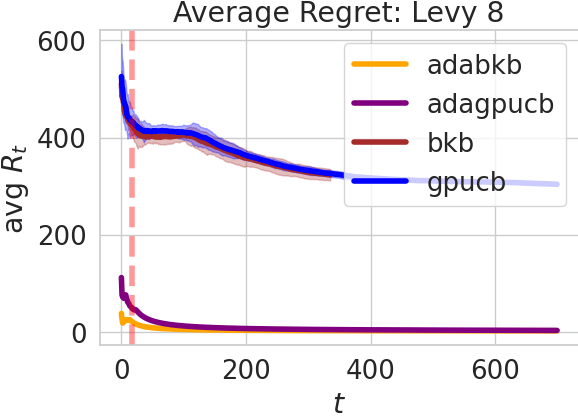}
	\includegraphics[width=0.24\linewidth]{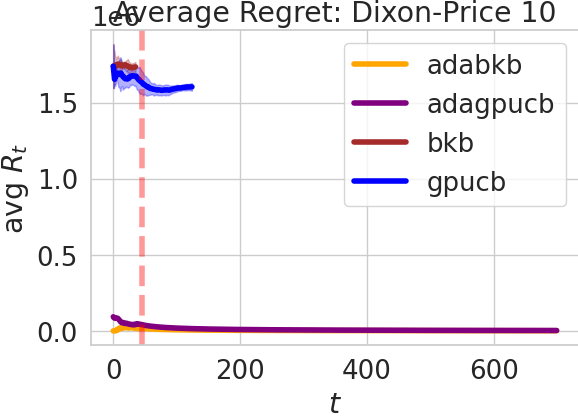}
	\includegraphics[width=0.24\linewidth]{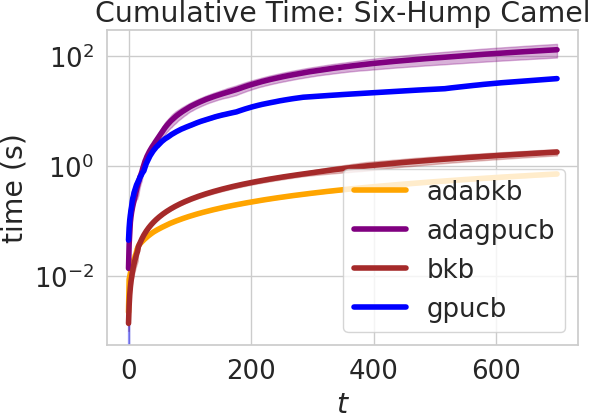}
	\includegraphics[width=0.24\linewidth]{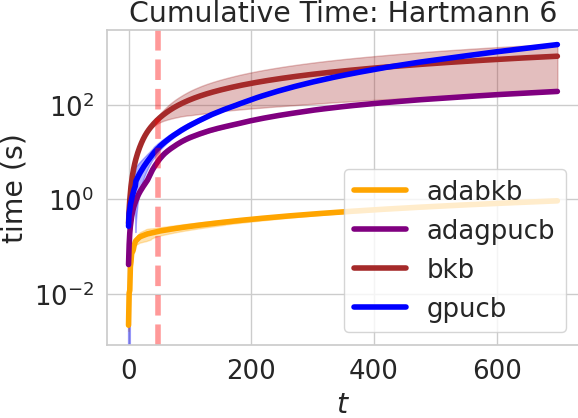}
	\includegraphics[width=0.24\linewidth]{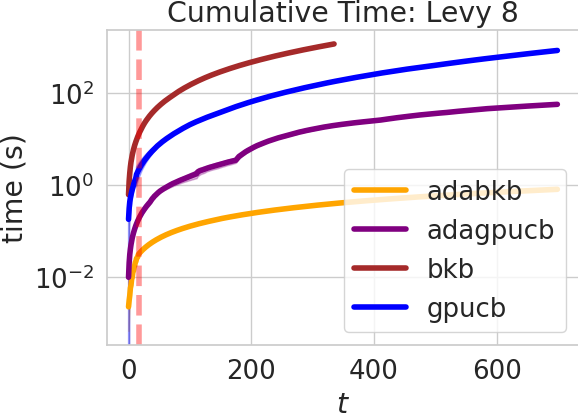}
	\includegraphics[width=0.24\linewidth]{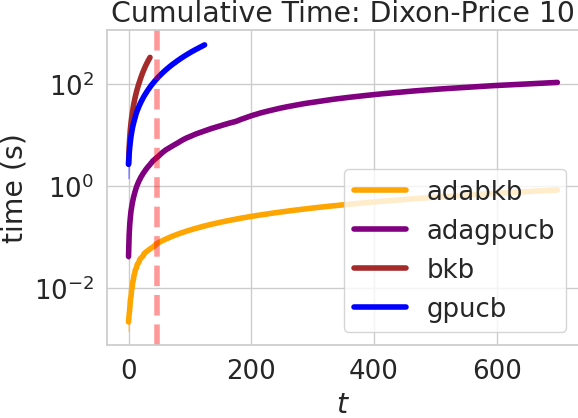}

	\caption{from left to right, average regret (first line) and cumulative time (second line) 
		obtained by algorithms in optimizing, from top to bottom, Six-Hump Camel, Hartmann 6, Levy 8 and Dixon-Price 10 functions.}
	\label{fig:main_paper_fig1}
\end{figure}
\begin{figure}[H]
	\centering
	\includegraphics[width=0.24\linewidth]{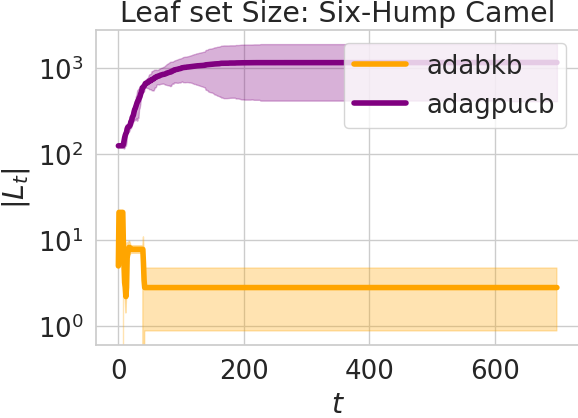}
	\includegraphics[width=0.24\linewidth]{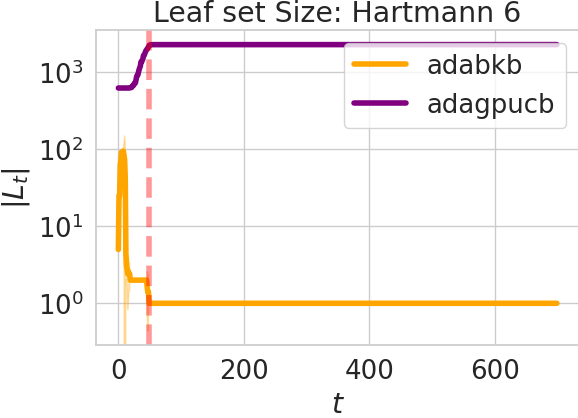}
	\includegraphics[width=0.24\linewidth]{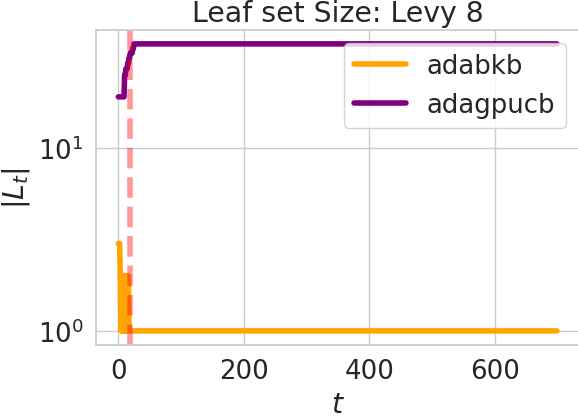}
	\includegraphics[width=0.24\linewidth]{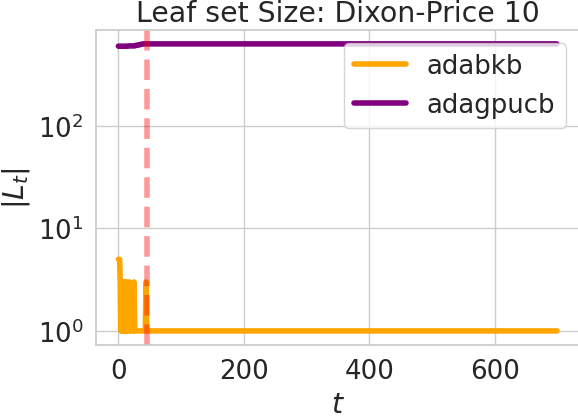}
	\caption{from left to right leaf set size of the algorithms in optimizing Six-Hump Camel, 
		Hartmann 6, Levy 8 and Dixon-Price 10 functions.}
	\label{fig:main_paper_fig2}
\end{figure}
For a budget $T$, in Figure \ref{fig:main_paper_fig1} we show the average regret and the cumulative time per function evaluation. 
In Figure~\ref{fig:main_paper_fig2} we show the leaf set size per iteration for Ada-BKB and AdaGP-UCB. 
We added a time threshold of 600 seconds. 
The red vertical line in Figure~\ref{fig:main_paper_fig1} and~\ref{fig:main_paper_fig2}, if present, 
indicates the (mean) iteration in which the early stopping condition is satisfied. 
We do not interrupt the execution just to show the behaviour of the algorithm 
(as you can notice in leaf set size plots, after the red line leaf set of Ada-BKB has cardinality $1$).
From second column of Figure~\ref{fig:main_paper_fig1}, we immediately note that 
AdaGP-UCB and Ada-BKB scale better with the search space dimension, but 
for low dimensional spaces (as Six-Hump Camel) AdaGP-UCB is more time consuming than GP-UCB. 
This is because for small dimensions we used small discretizations (15 points per dimension, see Appendix~\ref{app:experiements}) and hence the computations 
to build the matrices are cheap, 
while for the adaptive discretization 
we have to perform the expansion procedure. 
This is not necessarily always true for Ada-BKB thanks to the pruning procedure 
that let us balance the cost of expansion with the cost of evaluating the index. 
More experiments in Appendix~\ref{app:other_exp}.

\paragraph*{Hyper-parameter tuning.}
We performed experiments to tune the hyper-parameters of a 
recently proposed large scale kernel method~\citep{rudi2017falkon}. 
We compared Ada-BKB with AdaGP-UCB and BKB in minimizing 
the target function $f$ which takes as inputs a set of hyper-parameters 
to compute a hold-out cross-validation estimate of the error using 40\% of the data. 
The method in~\citep{rudi2017falkon} is based on a \Nystrom{} approximation of 
kernel ridge regression. 
In our experiments, we used a Gaussian kernel $k$ and tuned a lengthscale parameters $\sigma_1,\cdots,\sigma_p$ in each of the $p$ input dimensions. Indeed, we fixed the ridge parameters and the centers of the \Nystrom{} approximation 
(see Appendix~\ref{app:hpo_details} for details). 
We considered also BKB on a random discretization of size equal to 
the number of points evaluated by Ada-BKB, called {Random-BKB} in the following.
Again, for each experiment, we report mean and $95\%$ confidence interval using 5 repetitions. 
We added a time-threshold of 20 minutes.
\begin{figure}[H]
	\centering
	\includegraphics[width=0.32\linewidth]{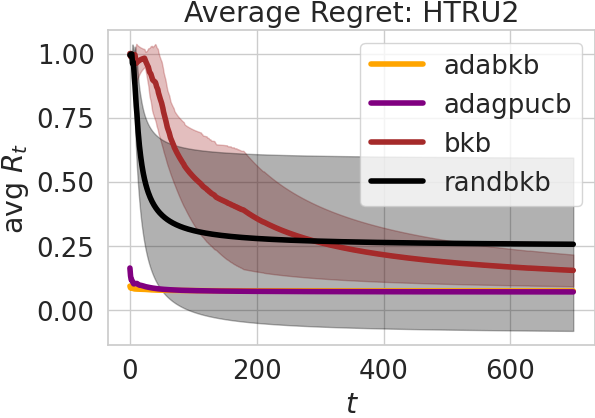}
	\includegraphics[width=0.32\linewidth]{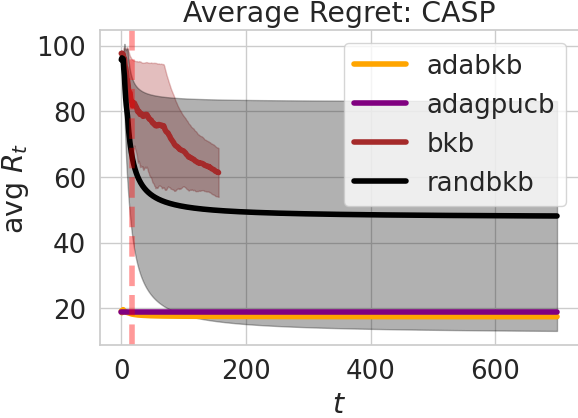}
	\includegraphics[width=0.32\linewidth]{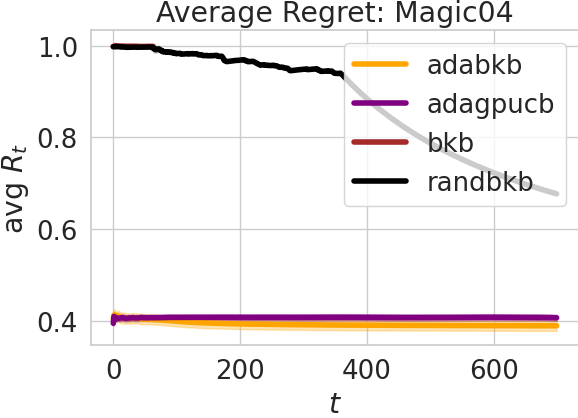}
	\includegraphics[width=0.32\linewidth]{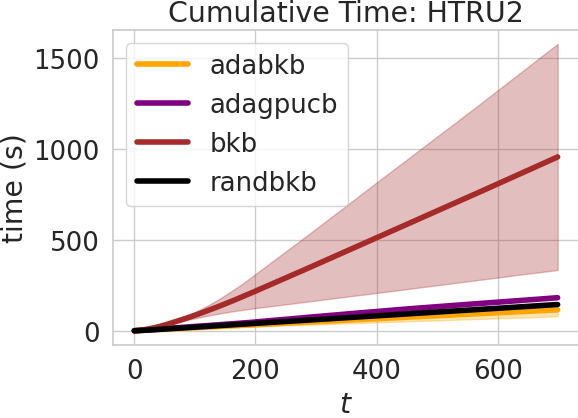}
	\includegraphics[width=0.32\linewidth]{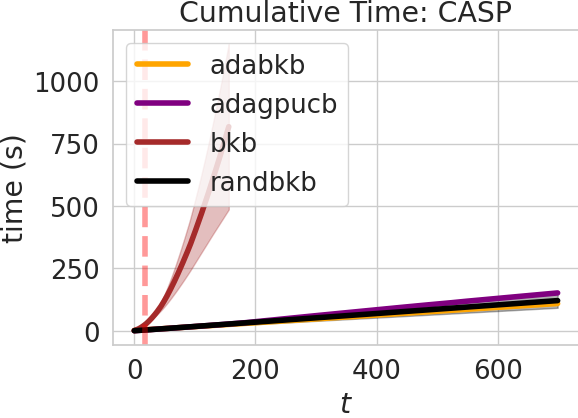}
	\includegraphics[width=0.32\linewidth]{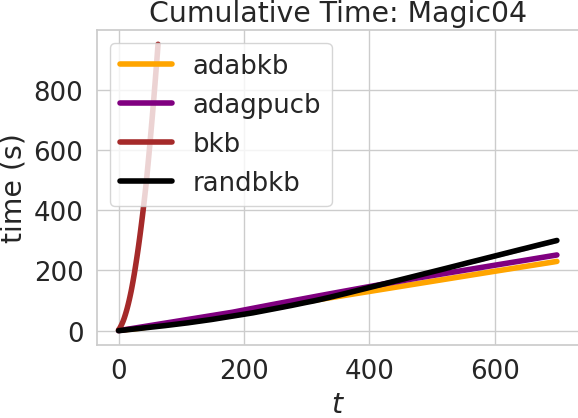}
	\caption{Average regret and cumulative time in optimizing the target function on HTRU2, CASP and Magic04 dataset.}
	\label{fig:flk_reg_time}
\end{figure}
\begin{figure}[H]
	\centering
	\vspace{.3in}
	\includegraphics[width=0.32\linewidth]{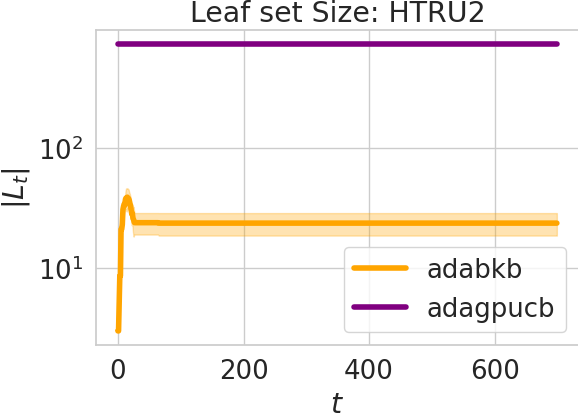}
	\includegraphics[width=0.32\linewidth]{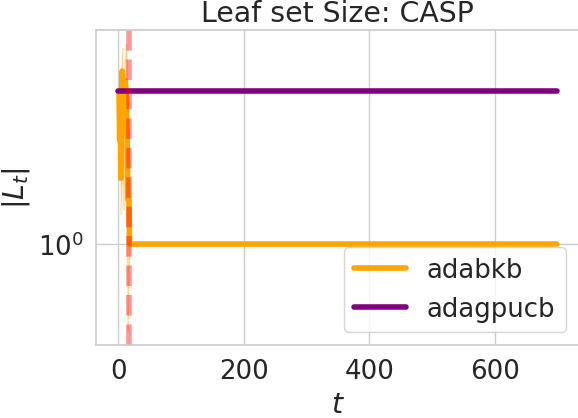}
	\includegraphics[width=0.32\linewidth]{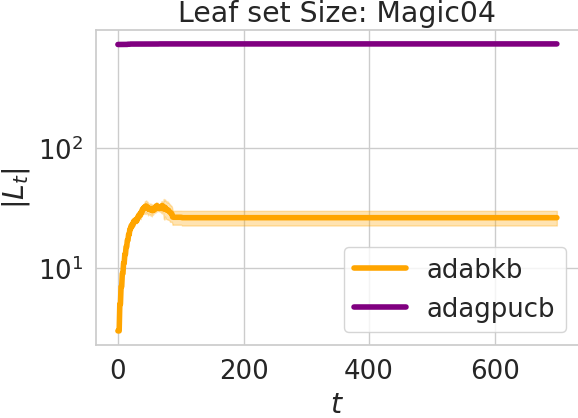}
	\vspace{.3in}
	\caption{Leaf set size in optimizing the target function on HTRU2, CASP and Magic04 dataset.}
	\label{fig:flk_lset_size}
\end{figure}

In Figure~\ref{fig:flk_reg_time}, we note that Ada-BKB obtains smaller or similar regret to other algorithms. 
In terms of time, Ada-BKB is typically the fastest method. 
In some cases, we note that Random BKB can obtain similar time performance than Ada-BKB, but typically the regret is larger, see e.g. the first line of Figure~\ref{fig:flk_reg_time}. 
Finally, we report the test error obtained fitting the model with the hyper-parameter configuration found by Ada-BKB and 
the time nedeed to perform every function evaluation until the budget or the time threshold is reached.
\begin{table}[H]
	\caption{Mean $\pm$ standard deviation of test error (MSE) using the configuration found by the algorithms with $5$ repetition}
	\begin{center}
		\begin{tabular}{l l l l} %

		\textbf{ALGORITHM} & \textbf{HTRU2} & \textbf{CASP} & \textbf{MAGIC04} \\\hline\\ 
		BKB & $0.067 \pm 0.004$ & $33.67 \pm 17.79$ & $0.99 \pm 0.0005$\\
		Random BKB & $0.24 \pm 0.34$ & $47.79 \pm 35.81$ &  $0.412 \pm 0.01$\\
		Ada-BKB & \textbf{0.068 $\pm$ 0.005} & \textbf{17.07 $\pm$ 0.09}& \textbf{ 0.383 $\pm$ 0.014}\\
		AdaGPUCB & $0.071 \pm 0.003$ & $18.65 \pm 0.34$ &  $0.389 \pm 0.010$
	\end{tabular}
\end{center}
\end{table}
\begin{table}[H]
\caption{Mean $\pm$ standard deviation of time (seconds) used for perform every function evaluation or before interruption with $5$ repetition}
\begin{center}
	\begin{tabular}{l l l l}%

	\textbf{ALGORITHM} & \textbf{HTRU2} & \textbf{CASP} & \textbf{MAGIC04}\\\hline\\ %
	BKB & $956.26 \pm 622$ & $818.21 \pm 332$ & $ 950.98 \pm 1.30$\\
	Random BKB & $144.47 \pm 5.09$ & $120.82 \pm 28.85$ & $299.63 \pm 5.05$ \\
	Ada-BKB & \textbf{115.21 $\pm$ 35.65} & \textbf{109.12 $\pm$ 1.09} & \textbf{230.06 $\pm$ 3.61} \\
	AdaGPUCB & $181.91 \pm 1.81$ & $151.57 \pm 0.59$ &$251.48 \pm 0.65$ 
\end{tabular}  
\end{center}

\end{table}
\section{CONCLUSION}\label{sec:conclusion}
In this paper, we presented a scalable approach to Gaussian Process optimization on continuous domains, combining ideas from BKB and optimistic optimization. The proposed approach is analyzed theoretically in terms of regret guarantees, showing that improved efficiency can be achieved with no loss of accuracy. Empirically we report very good performances on both simulated data and a hyper-parameter tuning task. Our work opens a number of possible research directions. For example, efficiency could be further improved using experimentation batching, see~\citep{calandriello2020near}. Another interesting question could be to extend the ideas in the paper to other way to define upper function estimates for example based on expected improvements~\citep{qin2017improving}. 

\subsubsection*{Acknowledgments}
This material is based upon work supported by the Center for Brains, Minds and Machines (CBMM), funded by NSF STC award CCF-1231216, and the Italian Institute of Technology. We gratefully acknowledge the support of NVIDIA Corporation for the donation of the Titan Xp GPUs and the Tesla k40 GPU used for this research.
L. R. acknowledges the financial support of the European Research Council (grant SLING 819789), the AFOSR projects FA8655-20-1-7028, FA9550-18-1-7009, FA9550-17-1-0390 and BAA-AFRL-AFOSR-2016-0007 (European Office of Aerospace Research and Development), and the EU H2020-MSCA-RISE project NoMADS - DLV-777826.
S. V. acknowledges the support of GNAMPA 2020: ``Processi evolutivi con memoria descrivibili tramite equazioni integro-differenziali"
\newpage

\bibliography{./bibliography.bib}

\clearpage
\appendix

\thispagestyle{empty}

\section{AUXILIARY RESULTS} 
In the following, we state the propositions and 
lemmas required to prove Theorem~\ref{thm:reg_bounds}. 
\begin{prop}
	\label{prop:upperbound}
	~\cite[App. D, Theorem 9]{calandriello2019gaussian}
	Let $\epsilon \in (0,1)$, $\delta \in (0,1)$, $\lambda > 0$, $F = \nor{f}_\mathcal{H}$
	and let $\bar{\alpha} = \frac{1 + \epsilon}{1 - \epsilon}$.
	Then, with probability at least $1 - \delta$ and for all $t > 0$:
	\begin{equation*}
		\ti{\mu}_t(x) - {\beta}_t \ti{\sigma}_t(x) \leq f(x) \leq \ti{\mu}_t(x) + {\beta}_t \ti{\sigma}_t(x)
	\end{equation*}
	with
	\begin{equation}\label{eqn:beta}
		\beta_t = 2\lambda^2 \sqrt{\bar{\alpha} \log(\kappa^2 t) \Big( \sum\limits_{s = 1}^{t} \ti{\sigma}_t^2(x_s) \Big) + \log (1/\delta)} + \Big( 1 + \frac{1}{\sqrt{1 - \epsilon}}\Big)\sqrt{\lambda}F
	\end{equation}
\end{prop}
We show that the index function $I_t(\cdot)$ (eq.~\eqref{eqn:index}) 
is an upper bound on the maximum value of the function $f$ in a cell:

\begin{prop}[Upper bound on maximum of the function $f$]\label{prop:max_upperbound}
	Supposing Assumption~\ref{asm:smoothd_k} holds and assuming $f \in \hh_k$, let $f(x^{*}_{h,i})$ be the maximum of $f$ in cell $X_{h,i}$ and let $x_{h, i}$ be a point in the same cell. For an arbitrary number of children per cell $N \geq 1$, setting $\beta_t$ as defined in Proposition~\ref{prop:upperbound} and with $V_h$ defined in equation~\eqref{eqn:vh}, with probability at least $1 - \delta$, for all $h \geq 0$, $1 \leq i \leq N^h$ and for all $t > 0$, we have:
	\begin{equation*}
		f(x^{*}_{h,i}) \leq I_t(x_{h,i})
	\end{equation*}
	with $I_t(\cdot)$ index function defined in~\eqref{eqn:index}
\end{prop}
\begin{proof}
	Let $p$ be the parent function of $(X_h)_{h \in \mathbb{N}}$. For all 
	$t > 0$, the index function $I_t$ is defined as follow:
	\begin{equation*}
		I_t(x_{h,i}) = \min \{ \tilde{\mu}_t(x_{h,i}) + \beta_t \tilde{\sigma}_t(x_{h,i}), \tilde{\mu}_t(x_{p(h,i)}) + \beta_t \tilde{\sigma}_t(x_{p(h,i)}) +V_{h - 1}\} + V_h
	\end{equation*}
	From the definition of $V_h$ (see equation~\eqref{eqn:vh}), 
	for all $h \geq 0$ and $1 \leq i \leq N^h$:
	\begin{equation*}
		|f(x) - f(x^\prime) | \leq \nor{f}_kd_k(x, x^\prime) \leq V_h \qquad \forall x, x^\prime \in X_{h,i}
	\end{equation*}
	where $d_k$ is defined in Assumption~\ref{asm:smoothd_k}. 
	Let $x^{*}_{h,i}$ be the maximizer of $f$ in cell $X_{h,i}$ and let $x_{h,i}$ be any point in $X_{h,i}$.  
	It follows that $\forall h \geq 0$ and $1 \leq i \leq N^{h}$:
	\begin{equation*}
		f(x^{*}_{h,i}) \leq f(x_{h,i}) + V_h
	\end{equation*}
	Using Proposition~\ref{prop:upperbound} to upper bound $f(x^{*}_{h,i})$, it follows
	\begin{equation*}
		f(x^{*}_{h,i}) \leq \tilde{\mu}_t(x_{h,i}) + \beta_t\tilde{\sigma}_t(x_{h,i}) + V_h
	\end{equation*}
	for all $t > 0$ (with probability at least $1 - \delta$). 
	For the same reason and by construction of the partition tree (Definition~\ref{def:wellbehaved}), we have:
	\begin{equation*}
		f(x^{*}_{h,i}) \leq \tilde{\mu}_t(x_{p(h,i)}) + \beta_t\tilde{\sigma}_t(x_{p(h,i)}) + V_{h - 1}
	\end{equation*}
	where $V_{h -1}$ is an upper bound of the function variation at level $h - 1$. 
	Since $V_h \geq 0$,  
	\begin{equation*}
		f(x^{*}_{h,i}) \leq \tilde{\mu}_t(x_{p(h,i)}) + \beta_t \tilde{\sigma}_t(x_{p(h,i)}) + V_{h - 1} + V_h
	\end{equation*}
\end{proof}
\begin{remark}
	Note that for the root cell $(0,1)$ the parent function is not defined. In this case, the index function is defined as:
	\begin{equation*}
		I_t(x_{0,1}) = \ti{\mu}_t(x_{0,1}) + \beta_t \ti{\sigma}_t(x_{0,1}) + V_0
	\end{equation*}
\end{remark}
Let $x^{*}$ be a global maximizer of the function $f$ and suppose $x^* \in X_{h,i^{*}}$. 
Let $x_{h,i^{*}}$ be the centroid of $X_{h,i^{*}}$. 
Then, Proposition~\ref{prop:max_upperbound} implies that with probability at least $1 - \delta$, %
\begin{equation*}
	f(x^{*}) \leq I_t(x_{h,i^{*}})
\end{equation*}
Now, we procede providing an upper-bound $U_V$ of the ratio $\frac{V_{h}}{V_{h+1}}$ described by the following Proposition.
\begin{prop}\label{prop:ratio_upperbound}
	Suppose Assumption~\ref{asm:bound_g} holds and 
	set $h_0 = \frac{\log(\delta_k/v_1)}{\log(\rho)}$. 
	For all $h \geq 0$, %
	\begin{equation}
		\frac{V_{h}}{V_{h+1}} \leq \max \Big\{ \max\limits_{0 \leq h \leq h_0 - 1} \frac{V_{h}}{V_{h+1}}, \frac{C^\prime_k}{C_k} \rho^{-\alpha} \Big\} =: U_V
	\end{equation}
\end{prop}
\begin{proof}
	Using the definition of $V_h$ (Equation~\eqref{eqn:vh}), we can write the ratio as:
	\begin{equation*}
		\frac{V_h}{V_{h+1}} = \frac{Fg(v_1\rho^h)}{Fg(v_1\rho^{h+1})} = \frac{g(v_1 \rho^h)}{g(v_1 \rho^{h + 1})}
	\end{equation*}
	Now, we have that $\exists \delta_k > 0$ such that:
	\begin{equation*}
		C_k v_1^\alpha \rho^{h\alpha} \leq g(v_1\rho^h) \leq C^\prime_k v_1^\alpha \rho^{h\alpha} \qquad \forall v_1\rho^h \leq \delta_k
	\end{equation*}
	then for all $v_1\rho^h$ lower than $\delta_k$, we can write:
	\begin{align}
		\frac{V_h}{V_{h+1}} &= \frac{g(v_1 \rho^h)}{g(v_1 \rho^{h + 1})}\\
		&\leq \frac{ C^{\prime}_k v_1^\alpha \rho^{h\alpha}}{ C_k v_1^\alpha \rho^{h\alpha + \alpha}}\\
		&= \frac{C^\prime_k}{C_k} \frac{1}{\rho^\alpha} =\frac{C^\prime_k}{C_k}\rho^{-\alpha}
	\end{align}
	Now, to conclude the proof, it is enough to observe that in Assumption~\ref{asm:bound_g}
	\begin{equation*}
		(\forall h \geq h_0) \qquad v_1\rho^h \leq \delta_k
	\end{equation*}
	For $h < h_0$, we can upper bound the ratio $\frac{V_h}{V_{h+1}}$ just 
	with the maximum of the ratios for all $h \in [0, h_0 -1]$. So the statement follows.%
\end{proof}
Proposition~\ref{prop:ratio_upperbound} states that $\forall h \geq 0$ 
we have $V_h \leq U_VV_{h+1}$ and 
this fact is exploited in the following lemma 
which give us information about the points selected by the algorithm.
\begin{lemma}\label{lem:subopt}
	Suppose that Assumptions~\ref{asm:smoothd_k},\ref{asm:bound_g},\ref{asm:cell_radius} hold. 
	Set $\beta_t$ as in eq.~\eqref{eqn:beta}, define $V_h$ as in \eqref{eqn:vh}, 
	and let $f(x^{*})$ be the global maximum of $f$. 
	If at time $t$, $x_{h_t,i_t} \in L_\tau$ is \textbf{evaluated} 
	then with probability at least $1 - \delta$:
	\begin{equation*}
		f(x^{*}) - f(x_{h_t, i_t}) \leq (4U_V + 1)V_{h_t}
	\end{equation*}
	Moreover, if $h<h_{\text{max}}$ then 
	\begin{equation*}
		f(x^{*}) - f(x_{h_t, i_t}) \leq 3\beta_t\tilde{\sigma}_t(x_{h_t, i_t})
	\end{equation*}
\end{lemma}
\begin{proof}
	According to the Proposition~\ref{prop:upperbound}, 
	setting $\beta_t$ as in eq.~\eqref{eqn:beta}, we have that
	\begin{equation*}
		\tilde{\mu}_t(x) - \beta_t \tilde{\sigma}_t(x) \leq f(x) \leq \tilde{\mu}_t(x) + \beta_t \tilde{\sigma}_t(x)
	\end{equation*}
	with probability of $1 - \delta$. From equation~\eqref{eqn:vh}), we have for all $h \geq 0$ and $1 \leq i \leq N^h$:
	\begin{equation*}
		\sup\limits_{x_1, x_2 \in X_{h,i}} | f(x_1) - f(x_2) | \leq V_h
	\end{equation*}
	Suppose that at time $t$ $x^{*}$ is contained in the cell $X_{h^{*}_t, i^{*}_t}$ represented by $x_{h^{*}_t,i^{*}_t} \in L_\tau$
	and that the algorithm selects and evaluate the point $x_{h_t,i_t}$. 
	From Proposition~\ref{prop:max_upperbound}, with probability at least $1 - \delta$, we have that 
	\begin{equation}
		f(x^{*}) \leq I_t(x_{h^{*}_t,i^{*}_t}).
	\end{equation}
	Since the algorithm selected $x_{h_t,i_t}$, according to the selection rule (row 5 of Algorithm~\ref{alg:1}) it follows that
	\begin{equation}
		I_t(x_{h^{*}_t,i^{*}_t}) \leq I_t(x_{h_t,i_t}).
	\end{equation}
	We recall that $I_t$ is defined as:
	\begin{equation*}
		I_t(x_{h_t,i_t}) = \min\{ \tilde{\mu}_t(x_{h_t,i_t}) + \beta_t\tilde{\sigma}_t(x_{h_t,i_t}), \tilde{\mu}_t(x_{p(h_t,i_t)}) + \beta_t\tilde{\sigma}_t(x_{p(h_t,i_t)}) + V_{h_t - 1} \} + V_{h_t}
	\end{equation*}
	therefore
	\begin{equation}
		\label{eqn:in1}
		f(x^*) \leq I_t(x_{h^{*}_t,i^{*}_t}) \leq I_t(x_{h_t,i_t}) \leq \tilde{\mu}_t(x_{p(h_t,i_t)}) + \beta_t\tilde{\sigma}_t(x_{p(h_t,i_t)}) + V_{h_t - 1} + V_{h_t}.
	\end{equation}
	In the rest of the proof we upper bound the right hand side. 
	Proposition~\ref{prop:upperbound}, yields (with probability at least $1 - \delta$):
	\begin{equation*}
		f(x_{p(h_t,i_t)}) \geq \tilde{\mu}_t(x_{p(h_t,i_t)}) - \beta_t\tilde{\sigma}_t(x_{p(h_t,i_t)}),
	\end{equation*}
	and therefore
	\begin{equation}
		\tilde{\mu}_t(x_{p(h_t,i_t)}) + \beta_t\tilde{\sigma}_t(x_{p(h_t,i_t)}) + V_{h_t - 1} + V_{h_t} \leq f(x_{p(h_t,i_t)}) + 2\beta_t\tilde{\sigma}_t(x_{p(h_t,i_t)}) + V_{h_t - 1} + V_{h_t}
	\end{equation}
	Since the algorithm evaluated $x_{h_t,i_t}$, then $\tilde{\beta}\tilde{\sigma}_t(x_{p(h_t,i_t)}) \leq V_{h_t - 1}$ therefore
	\begin{equation}
		f(x_{p(h_t,i_t)}) + 2\beta_t\tilde{\sigma}_t(x_{p(h_t,i_t)}) + V_{h_t - 1} + V_{h_t} \leq (f(x_{p(h_t,i_t)}) + V_{h_t - 1}) + 2V_{h_t - 1} + V_{h_t} 
	\end{equation}
	By construction of the partition tree, $x_{h_t,i_t}$ lies in the cell associated to $x_{p(h_t,i_t)}$, and so $f(x_{p(h_t,i_t)}) \leq f(x_{h_t,i_t})+V_{h_t}$. Hence, 
	\begin{equation}
		(f(x_{p(h_t,i_t)}) + V_{h_t - 1}) + 2V_{h_t - 1} + V_{h_t} \leq f(x_{h_t,i_t}) + 4V_{h_t - 1} + V_{h_t} 
	\end{equation}
	and, using Proposition~\ref{prop:ratio_upperbound}:
	\begin{equation}
		f(x_{h_t,i_t}) + 4V_{h_t - 1} + V_{h_t} \leq   f(x_{h_t,i_t}) +(4U_V + 1)V_{h_t}.
	\end{equation}
	The latter combined with \eqref{eqn:in1},  implies that 
	\begin{equation}
		\label{eqn:lem1_ub1}
		f(x^*) \leq   f(x_{h_t,i_t}) +(4U_V + 1)V_{h_t}.
	\end{equation}  
	To prove the second bound of the statement, note that
	\begin{equation}
		I_t(x_{h^{*}_t,i^{*}_t}) \leq I_t(x_{h_t,i_t}) \leq \tilde{\mu}_t(x_{h_t,i_t}) + \beta_t\tilde{\sigma}_t(x_{h_t,i_t}) + V_{h_t}
	\end{equation}
	Proposition~\ref{prop:upperbound} yields that, with probability at least $1 - \delta$
	\begin{equation*}
		f(x_{h_t,i_t}) \geq \tilde{\mu}_t(x_{h_t,i_t}) - \beta_t\tilde{\sigma}_t(x_{h_t,i_t})
	\end{equation*}
	then it follows:
	\begin{equation}
		\tilde{\mu}_t(x_{h_t,i_t}) + \beta_t\tilde{\sigma}_t(x_{h_t,i_t}) + V_{h_t} \leq f(x_{h_t,i_t}) + 2\beta_t\tilde{\sigma}_t(x_{h_t,i_t}) + V_{h_t}
	\end{equation}
	Next, if $h < h_{\text{max}}$, since $x_{h_t,i_t}$ is evaluated, then $\beta_t\tilde{\sigma}_t(x_{h_t,i_t}) > V_{h_t}$, and
	\begin{equation}
		f(x^*)\leq  \tilde{\mu}_t(x_{h_t,i_t}) + 2\beta_t\tilde{\sigma}_t(x_{h_t,i_t}) + V_{h_t} \leq f(x_{h_t,i_t}) + 3\beta_t \tilde{\sigma}_t(x_{h_t,i_t})
	\end{equation}
	In conclusion, we derive that if $h<h_{\text{max}}$
	\begin{equation*}
		f(x^{*}) - f(x_{h_t,i_t}) \leq 3\beta_t \tilde{\sigma}_t(x_{h_t,i_t}).
	\end{equation*}
\end{proof}

\begin{prop}\label{prop:bkb_regret}
	~\cite[Theorem 2]{calandriello2019gaussian}
	For any desired $0 < \epsilon < 1$, $0 < \delta < 1$, $\lambda > 0$, let $\bar{\alpha} = \frac{1 + \epsilon}{1 - \epsilon}$. 
	For $\beta_t$ defined as :
	\begin{equation*}
		\beta_t = 2\lambda^2 \sqrt{\bar{\alpha} \log(\kappa^2 t) \Big( \sum\limits_{s = 1}^{t} \ti{\sigma}_t^2(x_s) \Big) + \log (1/\delta)} + \Big( 1 + \frac{1}{\sqrt{1 - \epsilon}}\Big)\sqrt{\lambda}F
	\end{equation*}
	and $\tilde{\sigma}_t$ defined as in equation~\eqref{eqn:mu_sig}
	we have that:
	\begin{equation*}
		3\tilde{\beta}_T \sum\limits_{t = 1}^{T}\tilde{\sigma}_t (x_t) \leq \mathcal{O}(\sqrt{T}d_\text{eff}(\lambda, \tilde{X}_T)\log T)
	\end{equation*}
	where $\tilde{X}_T$ is the set containing every centroid evaluated until timestep $T$.
\end{prop}

\begin{prop}[Standard deviation upper bound]\label{prop:std_upp}
	Consider the evaluation model $y_t = f(x_t) + \eta_t$ with $\eta_t \sim \mathcal{N}(0, \sigma^2)$, 
	and let be $n_t: X \rightarrow \N$ a function which given a centroid $x_{h,i}$ 
	returns the number of times that it has been evaluated until time step $t$.  
	For a desired $\epsilon \in (0,1)$, let $\bar{\alpha} = \frac{1 + \epsilon}{1 - \epsilon}$. 
	Then, if a centroid $x_{h,i}$ is evaluated $n_t(x_{h,i})$ times we have 
	\begin{equation*}
		\ti{\sigma}_{t}(x_{h,i}) \leq \sqrt{\bar{\alpha}}\frac{\sigma}{\sqrt{n_t(x_{h,i})}}
	\end{equation*}
\end{prop}
\begin{proof}
	\cite[Part 1 of Proposition 3]{shekhar2018gaussian} yields 
	\begin{equation*}
		\sigma_t(x_{h,i}) \leq \frac{\sigma}{\sqrt{n_t(x_{h,i})}}
	\end{equation*} 
	where $\sigma_t$ is defined as in eq.~\eqref{eqn:exact_mu_sig}. 
	\cite[Theorem 1]{calandriello2019gaussian} implies that 
	for a desired $\epsilon \in (0,1)$, 
	setting $\bar{\alpha} = \frac{1 + \epsilon}{1 - \epsilon}$, 
	$\tilde{\sigma}^2(x)$ defined in eq.~\eqref{eqn:mu_sig} satisfies 
	the following inequality:
	\begin{equation*}
		\tilde{\sigma}_t^2(x) \leq \bar{\alpha}\sigma_t^2(x)
	\end{equation*}
	Which gives%
	\begin{equation*}
		\tilde{\sigma}_t(x) \leq \sqrt{\bar{\alpha}}\sigma_t(x) \leq \sqrt{\bar{\alpha}} \frac{\sigma}{\sqrt{n_t(x)}}
	\end{equation*}
\end{proof}

\section{PROOFS OF MAIN RESULTS}\label{app:proofs}
In this appendix, we provide the proofs of Lemma~\ref{lm:val_vh} and Theorems~\ref{thm:reg_bounds}\ref{thm:comp_cost}.

\subsection{Proof of Lemma~\ref{lm:val_vh}}\label{proof:vh}  
For all $x,x'\in X_{h,i}$,
\begin{equation*}
	|f(x) - f(x')| = | \scal{f}{k(x,\cdot)-k(x', \cdot)} |\le \nor{f}d_k(x,x')\le \nor{f}g( d(x,x')) \le \nor{f}g( v_1\rho^h)
\end{equation*}

\subsection{Proof of Theorem~\ref{thm:reg_bounds}}\label{app:proof_thm2}
To prove the bound on the cumulative regret we need to introduce some objects.
First, denoting with $x_{h_t,i_t}$ the centroid of $X_{h_t,i_t}$ evaluated 
at function evaluation $t$, let's define $Q_T$ as the set containing every point 
evaluated at each function evaluation:
\begin{equation*}
	Q_T = \{ x_{h_t, i_t} | 1 \leq t \leq T \}
\end{equation*}
Now, we split $Q_T$ in two sets $Q_1, Q_2$ defined as follow:
\begin{equation}\label{eqn:q_sets}
	\begin{aligned}
		Q_1 &= \{ x_{h, i} \in Q_T | h < h_{\text{max}} \}\\
		Q_2 &= Q_T \setminus Q_1      
	\end{aligned}
\end{equation}
So, we consider separately the terms which contribute to the cumulative regret:
\begin{equation*}
	R_T = \sum\limits_{x \in Q_T} f(x^{*}) - f(x) = \sum\limits_{x \in Q_1} f(x^{*}) - f(x) + \sum\limits_{x \in Q_2} f(x^{*}) - f(x) = R_1 + R_2,
\end{equation*}
where
\begin{equation*}
	R_1 = \sum\limits_{x \in Q_1} f(x^{*}) - f(x) \qquad \text{and} \qquad R_2 = \sum\limits_{x \in Q_2} f(x^{*}) - f(x).
\end{equation*}
Let's start by bounding $R_2$. Using Lemma~\ref{lem:subopt}, we can upper-bound $R_2$ as:
\begin{equation*}
	R_2 = \sum\limits_{x \in Q_2} f(x^{*}) - f(x) \leq (4U_V + 1)V_{h_{\text{max}}} |Q_2|
\end{equation*}
The size of $Q_2$ can be trivially upper-bounded with the budget $T$:
\begin{equation*}\label{eqn:R2_bound}
	(4U_V + 1)V_{h_{\text{max}}} |Q_2| \leq (4U_V + 1)V_{h_{\text{max}}} T
\end{equation*}
	Noting that $h_{\text{max}} \geq h_0$, %
	Assumption~\ref{asm:bound_g} implies%
	\begin{equation*}
		(4U_V + 1) V_{h_{\text{max}}} T \leq (4U_V + 1) C^\prime_k v_1^{\alpha} \rho^{h_{\text{max}}\alpha} T \leq \mathcal{O}(\rho^{h_{\text{max}}\alpha} T)
	\end{equation*}
	Moreover, since $h_{\text{max}} \geq \frac{1/2 \log T}{\alpha \log 1/\rho}$, 
	\begin{equation*}
		\rho^{h_{\text{max}}\alpha} T \leq \sqrt{T \log T}
	\end{equation*}
	To upper-bound $R_1$, since $|Q_1| \leq T$, Lemma~\ref{lem:subopt} yields:
	\begin{equation*}
		R_1 = \sum\limits_{x \in Q_1} f(x^{*}) - f(x) \leq 3\sum\limits_{x_{h_t, i_t} \in Q_1} \tilde{\beta}_t\tilde{\sigma}_t(x) 
	\end{equation*} 
	Again, since $|Q_1| \leq T$, we get %
	\begin{equation*}
		3\sum\limits_{x_{h_t, i_t} \in Q_1} \tilde{\beta}_t\tilde{\sigma}_t(x_{h_t, i_t}) \leq 3 \sum\limits_{t=1}^{T} \tilde{\beta}_t\tilde{\sigma}_t(x_{h_t,i_t}) \leq 3\tilde{\beta}_T \sum\limits_{t=1}^{T} \tilde{\sigma}_t(x_{h_t,i_t})
	\end{equation*}
	Proposition~\ref{prop:bkb_regret} implies%
	\begin{equation*}
		3\tilde{\beta}_T \sum\limits_{t=1}^{T} \tilde{\sigma}_t(x_{h_t,i_t}) \leq \mathcal{O}(\sqrt{T}d_\text{eff}(T)\log T) 
	\end{equation*} 
	Summing $R_1$ and $R_2$:
	\begin{align*}
		R_T &= R_1 + R_2\\
		&\leq \mathcal{O}(\sqrt{T}d_\text{eff}(T)\log T + \sqrt{T \log T})\\
		&\leq \mathcal{O}(\sqrt{T}d_\text{eff}(T)\log T)
	\end{align*}
	Now assume that the evaluation model is
	\begin{equation*}
		y_t = f(x_t) + \eta_t \qquad \text{with }\eta_t \sim \mathcal{N}(0, \sigma^2) 
	\end{equation*}
	In this scenario, 
	We follow a similar proof strategy of~\cite[Proof of Lemma 1]{salgia2020computationally}. 
	Let $Q_1$ be the set of observed centroids at depth $h< h_\text{max}$ (eq.~\eqref{eqn:q_sets}) 
	and let $n_i$ be the number of times that the $i$-th centroid (in the set $Q_1$) has been evaluated.
	Let $J$ be the set containing the indices of distinct points evaluated at least one time 
	at depth $h < h_\text{max}$:
	\begin{equation*}
		J = \{j : n_i > 0 \}.
	\end{equation*}
	It follows $|J| \leq \frac{N^{h_\text{max}} - 1}{N - 1}$, 
	which corresponds to the case in which Ada-BKB evaluates every point 
	in the partition tree with maximum depth $h_\text{max} - 1$.
	Considering $x_i$ as the $i$-th centroid in $Q_1$, let's denote with $t_j$ the time in which 
	$x_i$ has been selected and evaluated for the $j$-th time at timestep $t$ i.e. for all $2 \leq j \leq n_i$,  
	at timestep $t_j$, the centroid $x_{i}$ has been evaluated $j - 1$ times. 
	By Proposition~\ref{prop:std_upp}, we have that
	\begin{equation*}
		\tilde{\sigma}_{t_j - 1}(x_{t_j}) \leq \sqrt{\alpha}\frac{\sigma^2}{\sqrt{j - 1}}.
	\end{equation*}
	The contribution of every point $x_j$ with $j \in J$ to the sum of approximate variances is upper bounded by
	\begin{equation}\label{eqn:contrib_sigma}
		1 + \sqrt{\alpha}\sigma^2_n\sum\limits_{i=1}^{n_j - 1}\frac{1}{\sqrt{i}}
	\end{equation} 
	Lemma~\ref{lem:subopt} implies
	\begin{equation*}
		R_1 = \sum\limits_{x \in Q_1} f(x^*) - f(x) \leq 3\tilde{\beta}_{T} \sum\limits_{t=1}^{T} \tilde{\sigma}_{t-1}(x_{(h_t, i_t)})
	\end{equation*}
	We derive from~\eqref{eqn:contrib_sigma} that%
	\begin{align*}
		R_1 &\leq 3\tilde{\beta}_{T} \sum\limits_{j \in J} \Bigg( (1 + \sqrt{\bar{\alpha}}\sigma^2) \sum\limits_{k=1}^{n_j - 1} \frac{1}{\sqrt{k}} \Bigg)\\
		&\leq 3\tilde{\beta}_{T} \sum\limits_{j \in J} \Bigg( (1 + 2\sqrt{\bar{\alpha}}\sigma^2) \sqrt{n_j - 1} \Big)\\
		&\leq 3\tilde{\beta}_{T} (1 + 2\sqrt{\bar{\alpha}}\sigma^2) \sum\limits_{j \in J} \sqrt{n_j}\\
	\end{align*}
	By Jensen's inequality, %
	\begin{align*}
		R_1 &\leq 3\tilde{\beta}_{T} (1 + 2\sqrt{\bar{\alpha}}\sigma^2) |J| \sqrt{\frac{1}{|J|}\sum\limits_{j \in J} n_j}\\
		&\leq 3\tilde{\beta}_{T} (1 + 2\sqrt{\bar{\alpha}}\sigma^2) \sqrt{ |J| T}\\
		&\leq 3\tilde{\beta}_{T} (1 + 2\sqrt{\bar{\alpha}}\sigma^2) \sqrt{ \frac{N^{h_\text{max}} - 1}{N - 1} T}\\ 
	\end{align*}
	\cite[Appendix D.2]{calandriello2019gaussian} implies that 
	\begin{equation*}
		\tilde{\beta}_T \leq 2\lambda\sqrt{d_\text{eff} \log(k^2T) + \log(1/\delta)} + (1 + \frac{1}{\sqrt{1 - \epsilon}})\sqrt{\lambda}F
	\end{equation*}
	Therefore,
	\begin{equation*}
		R_1 \leq \mathcal{O} \Bigg(\sqrt{Td_\text{eff}\log(k^2 T) \frac{N^{h_\text{max}} - 1}{N - 1}} + \sqrt{T\frac{N^{h_\text{max}} - 1}{N - 1}} \Bigg)
	\end{equation*}
	If we take $N > 1$ s.t. $\sqrt{\frac{N^{h_\text{max}} - 1}{N - 1}} < T$, 
	we derive from~\eqref{eqn:R2_bound} that
	\begin{equation*}
		\mathcal{O} \Bigg( \sqrt{Td_\text{eff} \log(k^2 T) \frac{N^{h_\text{max}} - 1}{N - 1}} \Bigg) 
	\end{equation*}
	Notice that $\sqrt{\frac{N^{h_\text{max}} - 1}{N - 1}}$ doesn't grow with $p$ (search space dimension) as $d_\text{eff}$.
\subsection{Proof of Theorem~\ref{thm:comp_cost}}\label{app:proof_comp_cost}
Let $j$ be the number of observations at a certain time step.
We analyze the sources of cost of Algorithm \ref{alg:1} to get the computational cost.
\paragraph*{Model update.} According to the algorithm, every time we evaluate the function (i.e. we observe $y = f(x) + \eta$), we update our model. With BKB~\citep{calandriello2019gaussian}, we know that an update consists in recomputing $\tilde{\mu}$, $\tilde{\sigma}$ and in "resparsificating" the approximation. As indicated in~\citep{calandriello2019gaussian}, the computational cost of performing these operations is $\mathcal{O}(Td_\text{eff}^2(T))$.
\paragraph*{Index computation.} The computation of the index is the most expensive operation, see~\citep{shekhar2018gaussian}. 
In order to get a similar analysis to AdaGP-UCB, we consider the total cost of computing $I_t$. 
Since the cost of evaluating $\tilde{\mu}_t,\tilde{\sigma}_t$ for a point is $\mathcal{O}(d_\text{eff}^2(T))$, let's analyze two different scenarios:
\begin{enumerate}
	\item Refinement steps: if we have expanded a node, we don't perform an update of the model, so we can compute the index only for the new nodes (i.e. we just compute the approximated mean and variance for new nodes). 
	Each refinement operation adds $N$ new points to the leaf set and remove the expanded node, thus, the overall computational cost is:
	\begin{equation*}
		\mathcal{O}(Td_\text{eff}^2(T)(N - 1)h_{\text{max}})
	\end{equation*}
	\item Evaluation steps: after an evaluation, we update our model and, we have to recompute the index for the entire leaf set. 
	In the worst case, the leaf set $L_\tau$ at time $t$ contains every representative point of the nodes of the partition tree at depth $h_{\text{max}}$ and 
	since the sub-tree of partition tree at depth $h_{\text{max}}$ (and at any $h \geq 0$) is a perfect $N$-ary tree, 
	\begin{equation*}
		|L_\tau| \leq N^{h_{\text{max}}}.
	\end{equation*}
	So, the overall computational cost is:
	\begin{equation*}
		\mathcal{O}(T d_\text{eff}^2(T) N^ {h_{\text{max}}}).
	\end{equation*}
\end{enumerate}
\paragraph*{Candidate selection.} The selection procedure consists in chosing the $x \in L_\tau$ which maximize $I_t$, i.e.:
\begin{equation*}
	\argmax\limits_{x \in L_\tau} I_t(x)
\end{equation*}
Ignoring the cost of computing the index (since we analyzed it in the previous point), we have to consider the cost of computing the argmax in case we did refinement steps or evaluation steps:
\begin{enumerate}
	\item Refinement steps: after refinement steps, the model is not changed so we can take the argmax of new nodes (since the previous maximizer was the expanded node) 
	and this costs $\mathcal{O}((N-1)h_{\text{max}}T)$.
	\item Evaluation steps: we have to perform an exhaustive search on the leaf set and, this will cost:
	\begin{equation*}
		\mathcal{O}(T N^{h_{\text{max}}})
	\end{equation*}
\end{enumerate}
\paragraph*{Search space refinement.} When $X \subset \mathbb{R}^p$, the refinement of a cell $X_{h,i}$ is performed by dividing it equally in $N$ parts along its longest side (see also~\citep{shekhar2018gaussian}). This operation involves specifying the centers and the $p$ side lengths of each of the $N$ new cells and is thus a $\mathcal{O}(pN)$ operation.
So the overall cost of search space refinement is:
\begin{equation*}
	\mathcal{O}(Th_{\text{max}}Np)
\end{equation*}
So, the total cost for the algorithm is $\mathcal{O}(T d_\text{eff}^2(T) N^{h_{\text{max}}} + Th_{\text{max}}Np)$ and thus, fixed $p$:
\begin{equation*}
	\mathcal{O}(T d_\text{eff}^2(T) N^{h_{\text{max}}})
\end{equation*}

\section{EXPERIMENT DETAILS}\label{app:experiements}
In this appendix, we describe the optimizer settings used to perform experiments presented in Section \ref{sec:experiments}, showing also other experiments performed.
Every experiment is realized in Python 3.6.9 using \textbf{sklearn}\citep{scikit-learn,sklearn_api}, \textbf{pytorch}\citep{paszke2017automatic}, \textbf{gpytorch}\citep{gardner2018gpytorch} and \textbf{numpy}\citep{harris2020array} libraries.\\
The implementation of BKB used can be found on GitHub at the following link \url{https://github.com/luigicarratino/batch-bkb}
\subsection{Synthetic experiments details}
Synthetic experiments consist in finding global minima in well-known function, 
in particular, we considered the following functions and search spaces:
\begin{table}[H]
	\centering
	\caption{Function used and relative search space considered for Ada-BKB and AdaGP-UCB.}
	\begin{tabular}{ l l }

		\textbf{FUNCTION} & \textbf{SEARCH SPACE} $X$\\\hline\\
		Branin & $[-5.0, 10.0] \times [0.0, 15.0]$\\
		Beale & $[-4.5, 4.5]^2$ \\
		Bohachevsky& $[-10.0, 190.0] \times [-180.0, 20.0]$\\
		Rosenbrock 2& $[-5.0, 10.0]^2$\\
		Six-Hump Camel & $[-2.0, 2.0] \times [-3.0, 3.0]$ \\
		Ackley 2 & $[-10.0, 52.768]^2$ \\
		Trid 2 & $[-4.0, 4.0]^2$\\
		Hartmann 3 &$[0.0, 1.0]^3$\\
		Trid 4 & $[-16.0, 16.0]^4$\\
		Shekel & $[0.0, 10.0]^4$\\
		Ackley 5 & $[-10.0, 52.768]^5$\\
		Hartmann 6 & $[0.0, 1.0]^6$\\
		Levy 6 & $[-10.0, 10.0]^6$\\
		Levy 8 & $[-10.0, 10.0]^8$\\
		Rastrigin 8 & $[-1.12, 5.12]^8$\\
		Dixon-Price 10 & $[-10.0, 10.0]^{10}$\\
		Ackley 30 & $[-10.0, 52.768]^{30}$\\
	\end{tabular}  
	\label{Tab:fun}
\end{table}
The parameter $\delta$ is set to $10^{-5}$ for every experiments.
\begin{table}[H]
	\centering
	\caption{Parameters of the optimizer used for experiments presented in Section \ref{sec:experiments} and Appendix \ref{app:other_exp}.}
	\begin{tabular}{ l l l l l}

		\textbf{FUNCTION} & $\sigma$ & $h_{\text{max}}$ & $N$ & $p$ \\\hline\\
		Branin & $0.5$ & $5$ & $3$ & $2$\\
		Beale & $1.0$ & $5$ & $3$ & $2$\\
		Bohachevsky & $1.70$ & $9$ & $3$& $2$\\
		Rosenbrock 2 & $0.70$ & $10$ & $11$ & $2$\\
		Six-Hump Camel & $0.5$ & $6$ & $5$ & $2$\\
		Ackley 2 & $3.5$ & $7$ & $3$ & $2$\\
		Trid 2 & $1.5$ & $7$ & $5$ & $2$\\
		Hartmann 3 & $0.5$ & $7$ & $3$ & $3$\\
		Trid 4 & $10.75$ & $7$ & $13$ & $4$\\
		Shekel& $1.75$ & $6$ & $9$ & $4$\\
		Ackley 5 & $5.0$ & $6$ & $3$ & $5$\\
		Hartmann 6 & $0.35$ & $5$ & $5$ & $6$\\
		Levy 6 & $5.0$ & $7$ & $5$ & $6$\\
		Levy 8 & $2.5$ & $7$ & $3$ & $8$\\
		Rastrigin 8 & $7.0$ & $10$ & $3$ & $8$\\
		Dixon-Price 10 & $2.0$ & $10$ & $5$ & $10$\\
		Ackley 30 & $20.50$ & $300$ & $3$ & $30$\\
	\end{tabular}  
	\label{Tab:set}
\end{table}
Detailed information about the test functions is available at the following website: \url{https://www.sfu.ca/~ssurjano/optimization.html}.\\
For every algorithm, we used a Gaussian kernel with lengthscale $\sigma$ specified in Table~\ref{Tab:set}. The noise standard deviation (indicated with $\lambda$) is set to $0.01$ for every experiment.
Values for other parameters (like the kernel lengthscale $\sigma$) specified in Table~\ref{Tab:set} are obtained using cross-validation (the value of $h_{\text{max}}$ is just the logarithm of the budget).\\
For GP-UCB and BKB, the discrete search space was built by taking $15$ points for every dimension and computing the Cartesian product. 
For "mid dimensional" cases (5 and 6 dimensions), the number of points per dimension taken is $10$ and for higher dimensional spaces $5$ points per dimension are taken .\\
The parameter $F$ is set to be $1$. 
\subsection{Hyper-parameter tuning experiments details}\label{app:hpo_details}
For FALKON hyper-parameter tuning experiments, we used the following datasets
\begin{table}[H]
	\centering
	\caption{Dataset used with number of features and search spaces considered}
	\label{tab:flk_dataset}
	\begin{tabular}{ l l l }

		\textbf{DATASET} & $p$ & \textbf{SEARCH SPACE} \\\hline\\
		HTRU2 & $8$ & $[0.0, 1.0]^8$ \\
		CASP & $9$ & $[0.0, 1.0]^9$\\
		Magic04 & $10$ & $[0.1, 10.0]^{10}$ \\
	\end{tabular}  
\end{table}
In following tables, for each dataset, we indicate the number of rows, size for the training and test part and we also indicate the value for $M$ and $\lambda$ (Falkon parameters) used:
\begin{table}[H]
	\centering
	\caption{Falkon fixed parameter per dataset used and size of dataset and relative training and test parts}
	\begin{tabular}{ l l l l l l }

		\textbf{DATASET} & \textbf{ROWS} & \textbf{TRAINING} & \textbf{TEST} & M & $\lambda$ \\\hline
		HTRU2 & $17897$ & $15216$ & $3804$ & $1000$ & $1e-5$ \\
		CASP & $45730$ & $32010$ & $13720$ & $2000$ & $1e-5$ \\
		Magic04 & $19020$ & $14317$ & $3580$ & $2000$ & $1e-6$ \\
	\end{tabular}  
	
	\label{tab:flk_fixed_params}
\end{table}
\begin{table}[H]
	\centering
	\caption{parameter for the optimizer used for parameter tuning experiments}
	\begin{tabular}{l l l l l l}

		\textbf{DATASET} & $\sigma$ & $\lambda$ & $h_{\text{max}}$ & $N$ & $\delta$ \\\hline\\
		HTRU2 & $10.0$ & $1e-9$ & $6$ & $3$ & $1e-5$\\
		CASP & $5.0$ & $1e-9$ & $7$ & $5$ & $1e-5$\\
		Magic04 & $5.0$ & $1e-9$ & $6$ & $3$ & $1e-5$\\
	\end{tabular}  
	\label{tab:opt_fixed_params}
\end{table} 
Again, the parameter $F$ is set to be $1$. We used a Gaussian kernel $k$ with many lengthscale parameters $\sigma_1,\cdots,\sigma_p$ with $p$ number of features of the dataset
\begin{equation*}
	k(x, x^\prime) = e^{- \frac{1}{2}x \Sigma^{-1} x^{\prime}} \qquad \Sigma = \begin{bmatrix}
		\sigma^2_1 & 0 & \cdots & 0\\
		0 & \sigma^2_2 & \cdots & 0\\
		\vdots & \vdots & \ddots & \vdots \\
		0 & 0 &\cdots & \sigma^2_p 
	\end{bmatrix}
\end{equation*}  
Target function $f$ used is the $70$-$30$ hold-out cross-validation which splits the training set in training and validation where:
\begin{enumerate}
	\item training part is composed by the $70\%$ of the points of the training set and it is used to fit the model.
	\item validation part is composed of the remaining $30\%$ of the points of the training set and it is used to test our model fitted with the training part.
\end{enumerate}
Before splitting the training set, it is shuffled. 
The metric used to evaluate the model is the mean square error (MSE) which, given $y$ corresponding labels of the validation part and $\tilde{y}$ the label predicted by the model (on the validation part) is defined as follow:
\begin{equation*}
	MSE(y,\tilde{y}) = \frac{1}{n}\sum\limits_{i=1}^n (y_i - \tilde{y}_i)^2
\end{equation*}
Thus, we want to minimize the function $f$ which takes a parameter configuration, performs the hold-out cross-validation, and returns the MSE.
Since, we don't know which is the best parameter configuration and how large is the minimum MSE, to compute the average regret we assume that let $x^{*}$ be the optimal configuration, then $f(x^{*}) = 0$. 
We don't expect that our algorithm finds this configuration (also because it could not exist) but this strategy allows us to see which algorithm get the highest performance.
As for synthetic experiments, the parameters of the optimizer (Table~\ref{tab:opt_fixed_params}) are set using the value suggested by the theory and using cross-validation (for the number of children per node $N$, kernel lengthscale $\sigma$, etc) when it wasn't possible.
Falkon library~\citep{falkonlibrary2020} used can be found at following url: \url{https://github.com/FalkonML/falkon} (in particular, since dataset used are small enough, to speed-up computations we used InCore Falkon~\citep{falkonlibrary2020}).
Dataset used to perform experiments are split in training and test part (described in Table~\ref{tab:flk_fixed_params}). 
Preprocessing mostly consisted of data standardization to zero mean and unit standard deviation and, when a dataset is used for binary classification, labels are set to be $-1$ and $1$ (for instance for Magic04 dataset where labels are 'g' and 'h').
Dataset used can be downloaded at the following links:
\begin{enumerate}
	\item HTRU2~\citep{Lyon_2016,Dua:2019}: \url{https://archive.ics.uci.edu/ml/datasets/HTRU2}
	\item CASP~\citep{Dua:2019}: \url{https://archive.ics.uci.edu/ml/datasets/Physicochemical+Properties+of+Protein+Tertiary+Structure}
	\item Magic04~\citep{Dua:2019}: \url{https://archive.ics.uci.edu/ml/datasets/MAGIC+Gamma+Telescope}
\end{enumerate}
For each dataset, we estimated also the evaluation time of the target function on random parameter configuration to get an idea about how much this target function is expensive in time:
\begin{table}[H]
	\centering
	\caption{mean $\pm$ standard deviation time of evaluating the target function $f$ with a random configuration with $50$ repetition}
	\begin{tabular}{c c }

		\textbf{DATASET} & \textbf{FUNCTION EVALUATION}\\\hline\\
		HTRU2 & $0.1877 \pm 0.4682s$ \\
		CASP & $0.2562 \pm 0.4565s$\\
		Magic04 & $0.1971 \pm 0.4565s$\\
	\end{tabular}
	\label{tab:flk_single_eval}
\end{table}
\subsection{Machines used for experiments}\label{app:machine}
In the following tables, we describe the features of the machine used to perform the experiments presented in Section~\ref{sec:experiments} and Appendix~\ref{app:other_exp}.
\begin{table}[H]
	\centering
	\caption{machine used to perform the experiments}
	\begin{tabular}{l l }
		\textbf{FEATURE} & \\\hline\\
		OS & Ubuntu 18.04.1 \\
		CPU(s) & $2 \times$ Intel(R) Xeon(R) Silver 4116 CPU\\
		RAM & $256$GB\\
		GPU(s) & $2 \times$ NVIDIA Titan Xp (12 GB RAM)\\
		CUDA version & $10.2$
	\end{tabular}
\end{table}
Further details of GPUs used can be found in the following links: \url{https://www.nvidia.com/en-us/titan/titan-xp/}
\subsection{Other experiments}\label{app:other_exp}
We performed other experiments in minimizing well-known functions specified in Table~\ref{Tab:fun}. 
Again, for showing better the results, we just plot the first $700$ evaluations. 
The red vertical dashed line indicates when the early stopping condition is satisfied. 
We added a time threshold of 600 seconds.
\begin{figure}[H]
	\centering
	\includegraphics[width=0.25\linewidth]{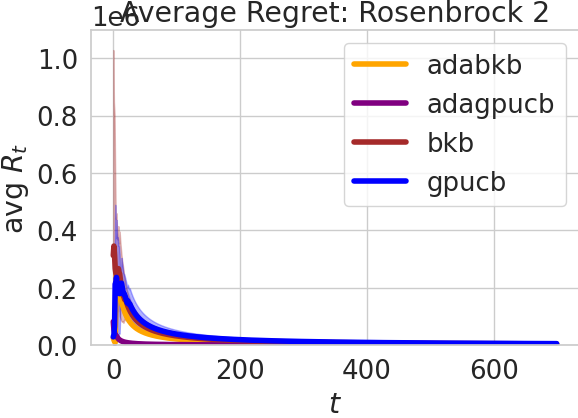}
	\includegraphics[width=0.25\linewidth]{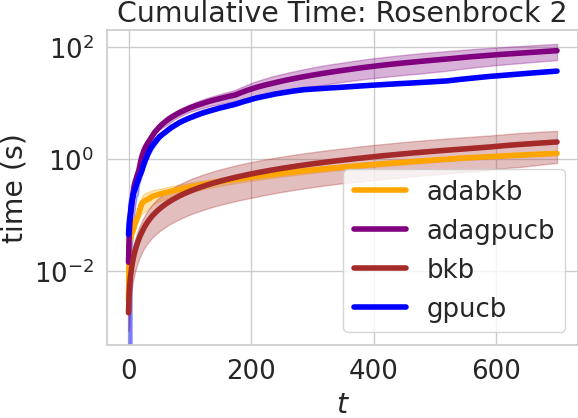}
	\includegraphics[width=0.25\linewidth]{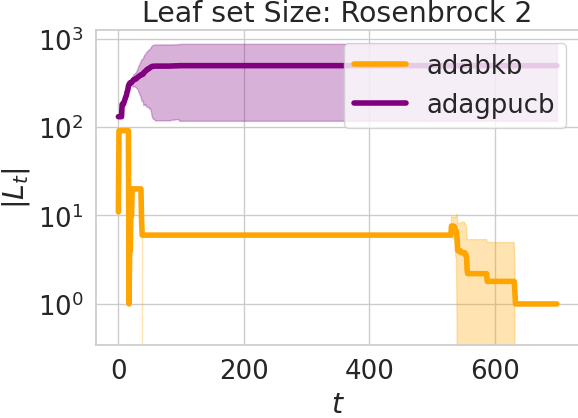}
	\includegraphics[width=0.25\linewidth]{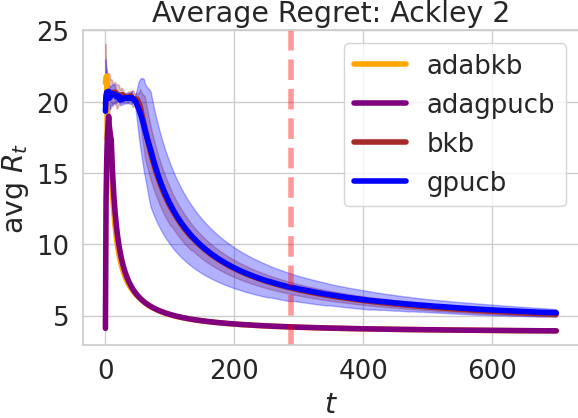}
	\includegraphics[width=0.25\linewidth]{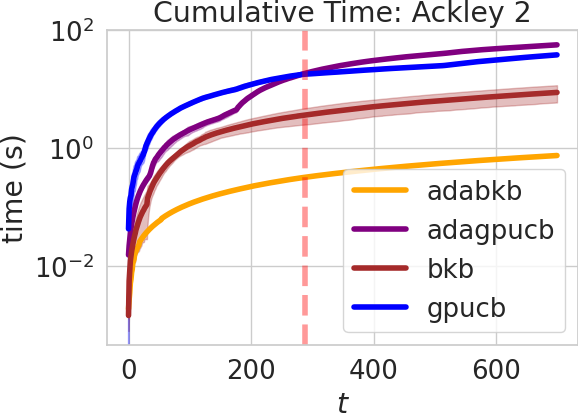}
	\includegraphics[width=0.25\linewidth]{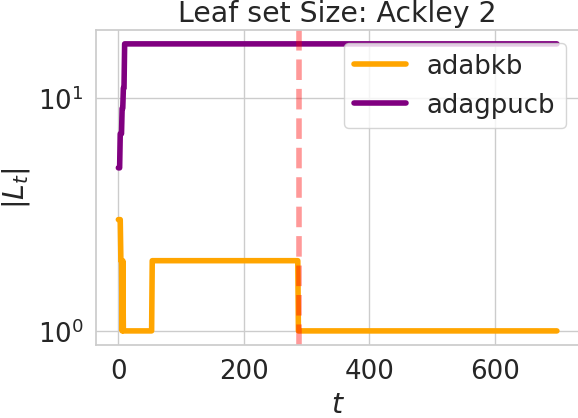}
	\includegraphics[width=0.25\linewidth]{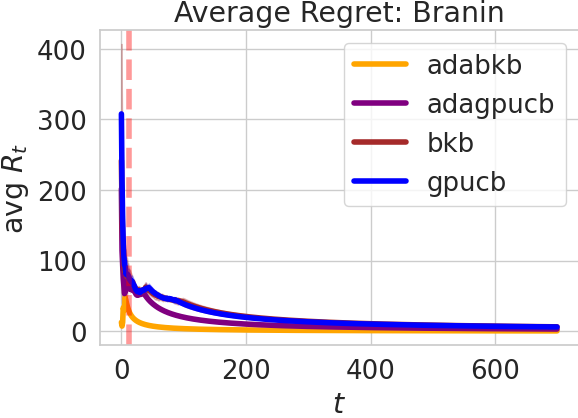}
	\includegraphics[width=0.25\linewidth]{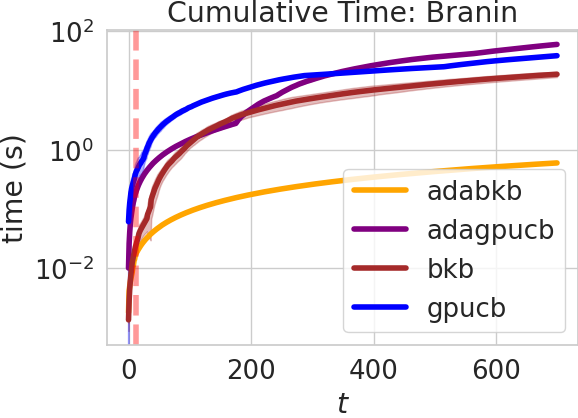}
	\includegraphics[width=0.25\linewidth]{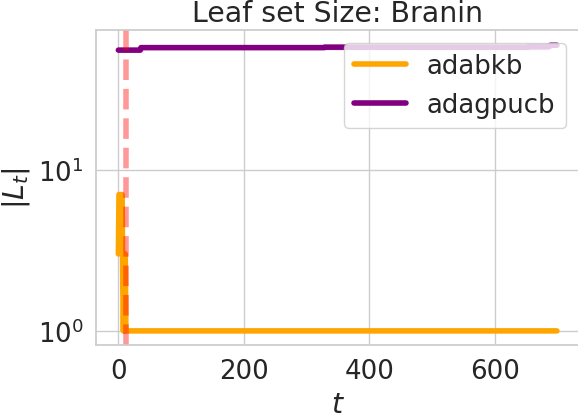}
	\includegraphics[width=0.25\linewidth]{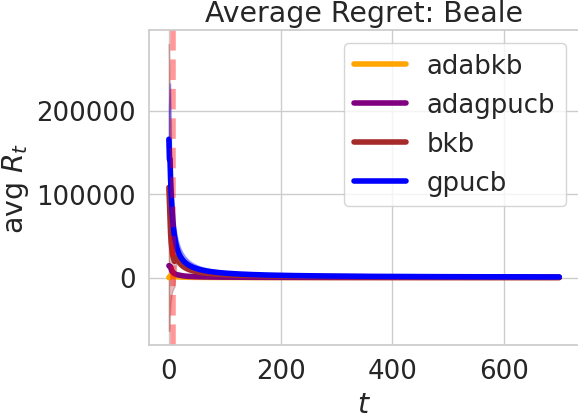}
	\includegraphics[width=0.25\linewidth]{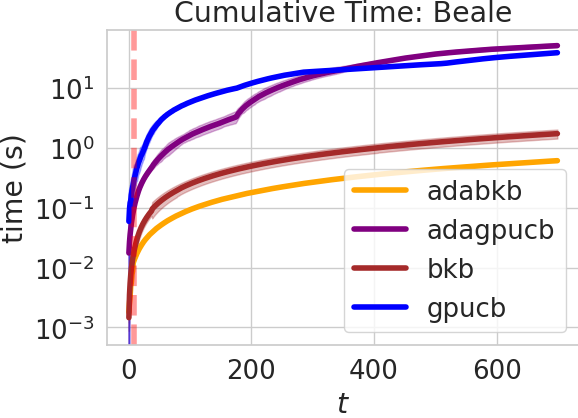}
	\includegraphics[width=0.25\linewidth]{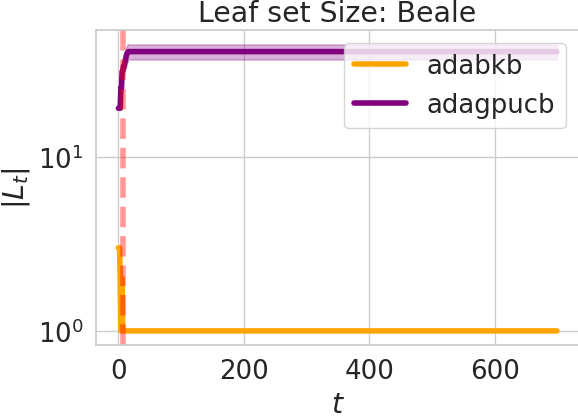}
	\includegraphics[width=0.25\linewidth]{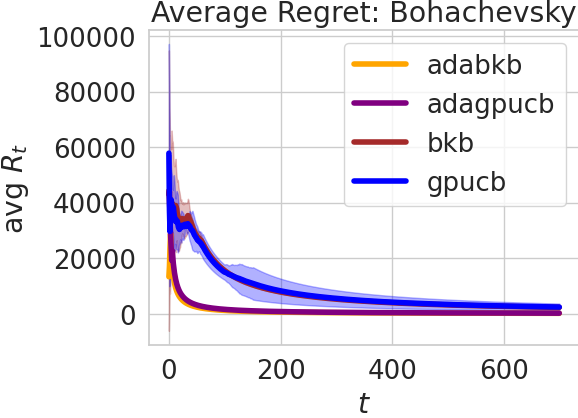}
	\includegraphics[width=0.25\linewidth]{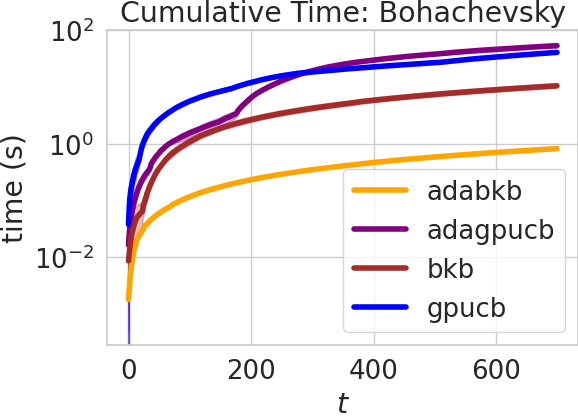}
	\includegraphics[width=0.25\linewidth]{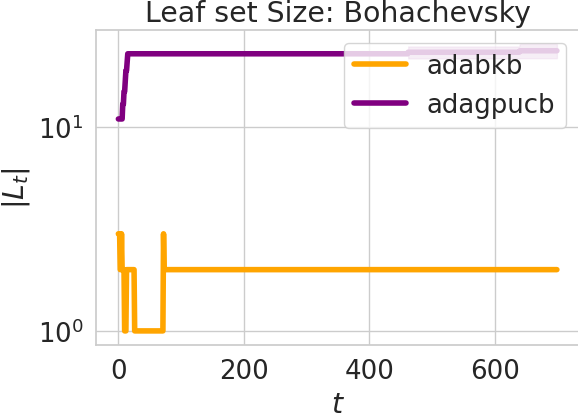}
	\includegraphics[width=0.25\linewidth]{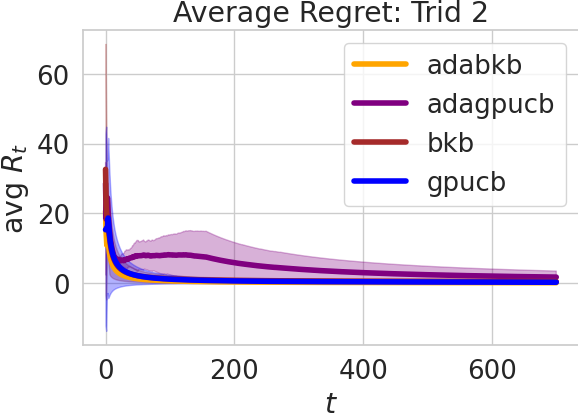}
	\includegraphics[width=0.25\linewidth]{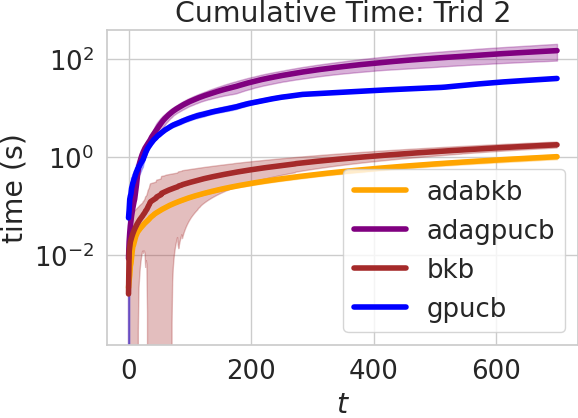}
	\includegraphics[width=0.25\linewidth]{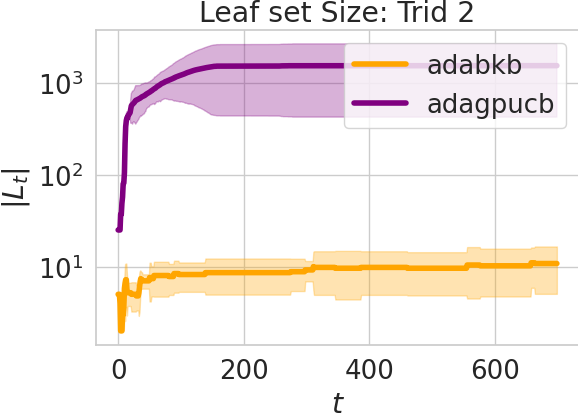}
	\caption{Average regret obtained by the algorithms in optimizing functions in Table \ref{Tab:set}}
	\label{fig:avg_reg_oexp_1}
\end{figure}
\begin{figure}[H]
	\centering
	\includegraphics[width=0.25\linewidth]{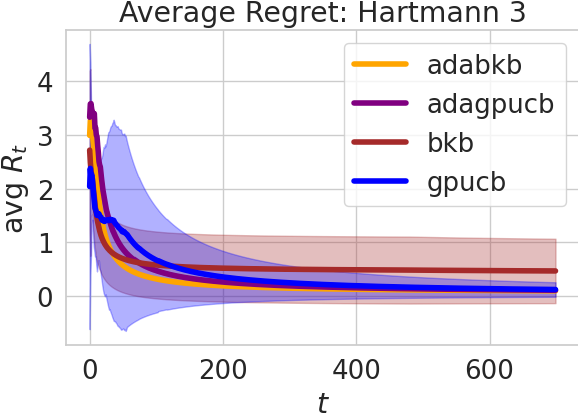}
	\includegraphics[width=0.25\linewidth]{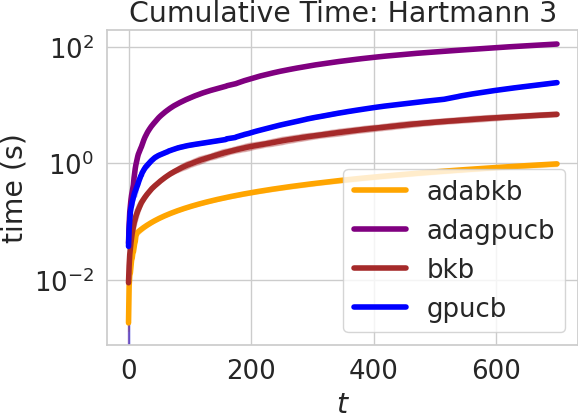}
	\includegraphics[width=0.25\linewidth]{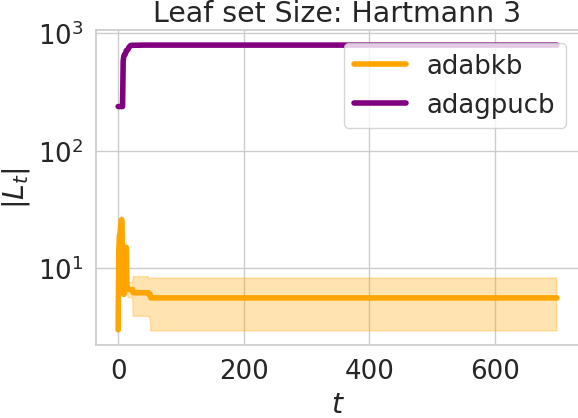}
	\includegraphics[width=0.25\linewidth]{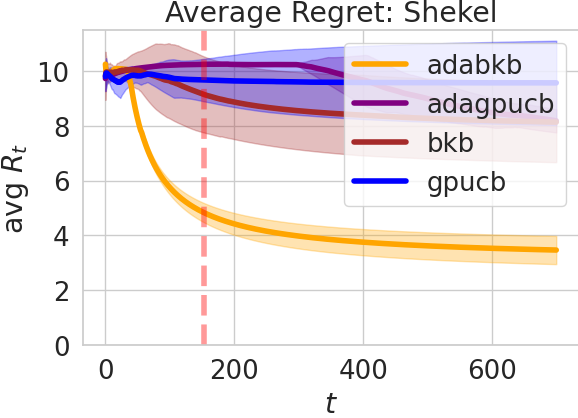}
	\includegraphics[width=0.25\linewidth]{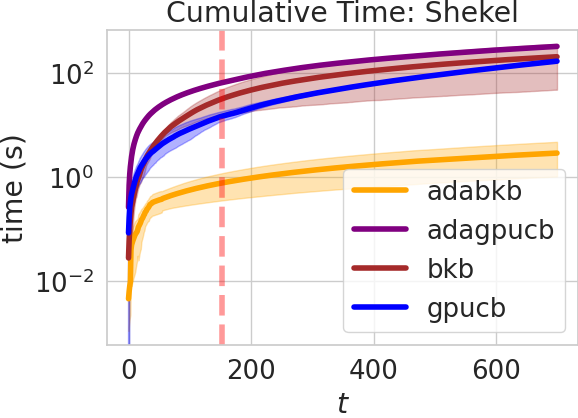}
	\includegraphics[width=0.25\linewidth]{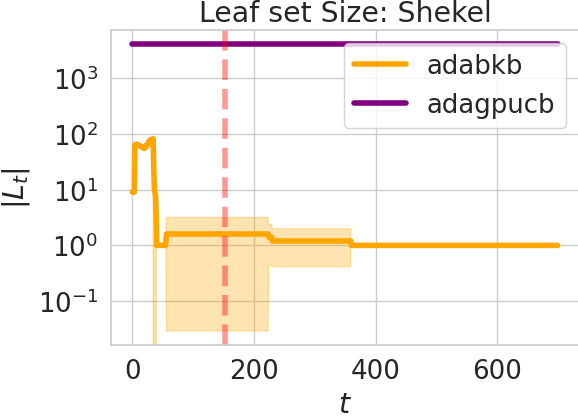}
	\includegraphics[width=0.25\linewidth]{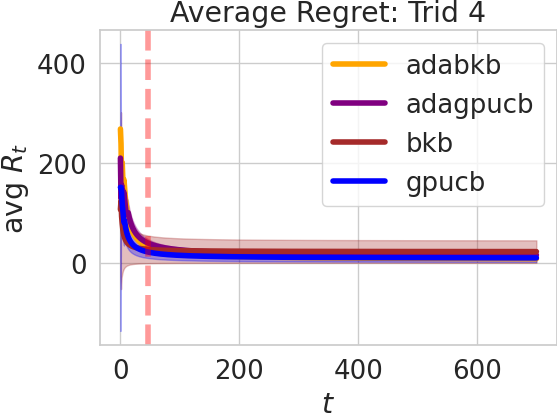}
	\includegraphics[width=0.25\linewidth]{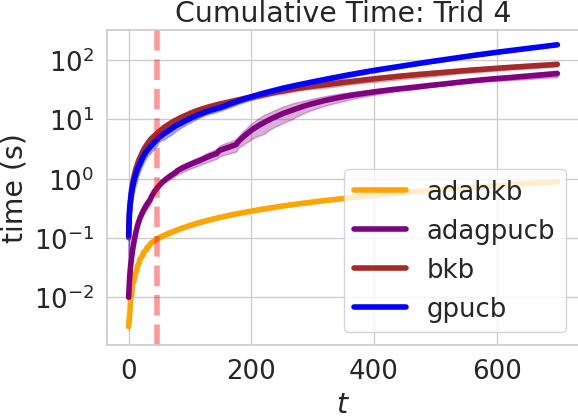}
	\includegraphics[width=0.25\linewidth]{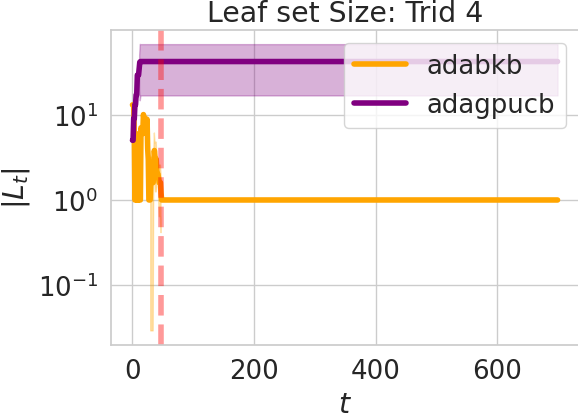}
	\includegraphics[width=0.25\linewidth]{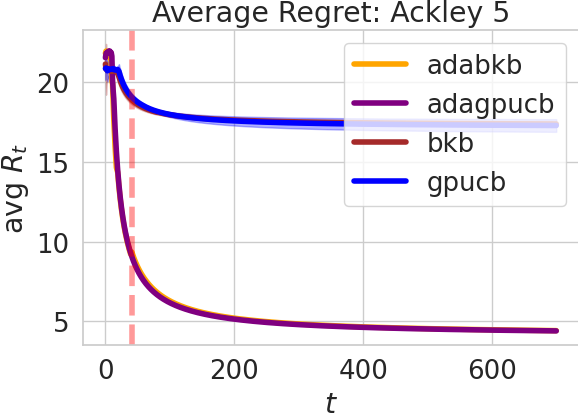}
	\includegraphics[width=0.25\linewidth]{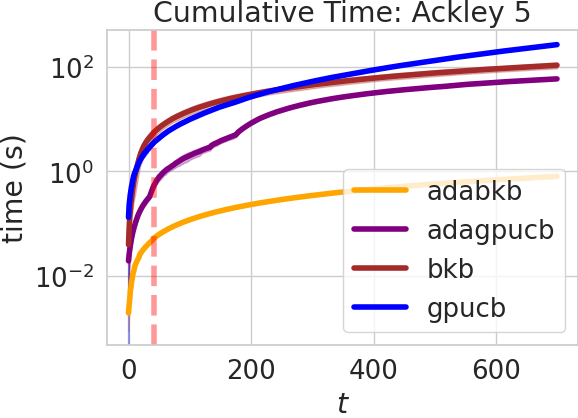}
	\includegraphics[width=0.25\linewidth]{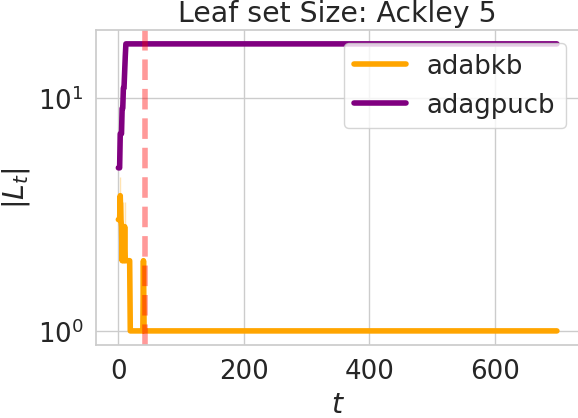}
	\includegraphics[width=0.25\linewidth]{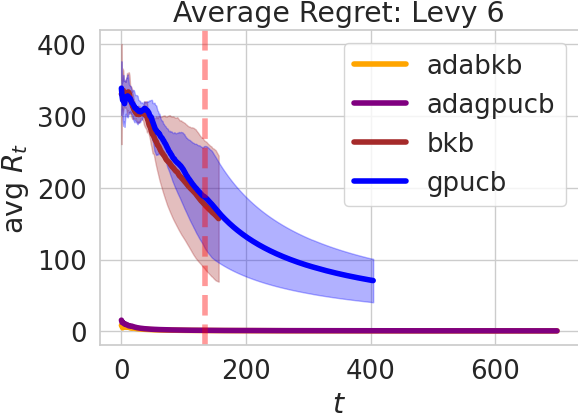}
	\includegraphics[width=0.25\linewidth]{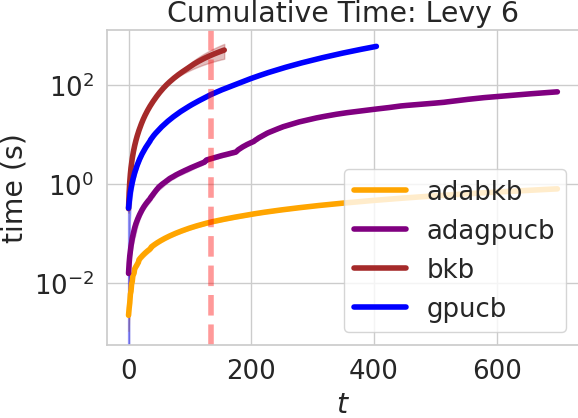}
	\includegraphics[width=0.25\linewidth]{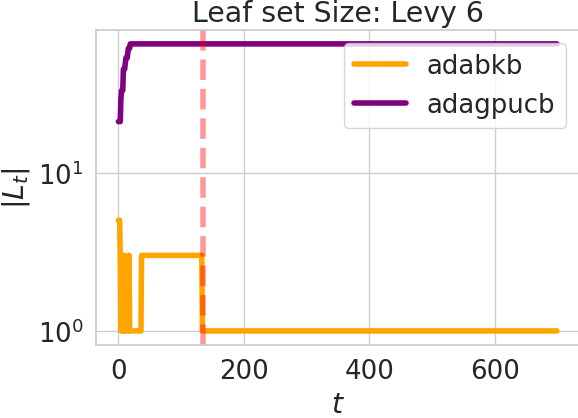}
	\includegraphics[width=0.25\linewidth]{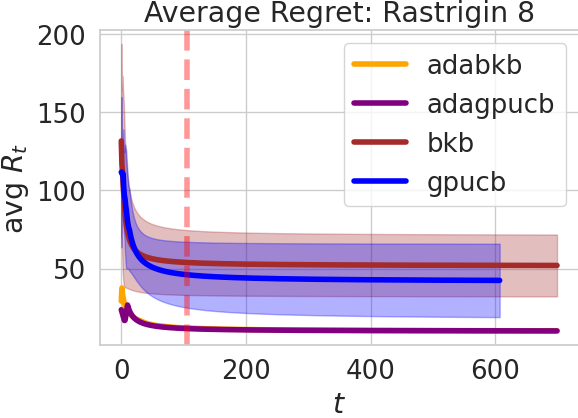}
	\includegraphics[width=0.25\linewidth]{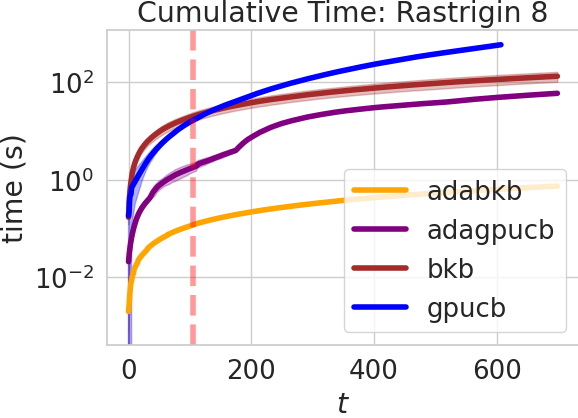}
	\includegraphics[width=0.25\linewidth]{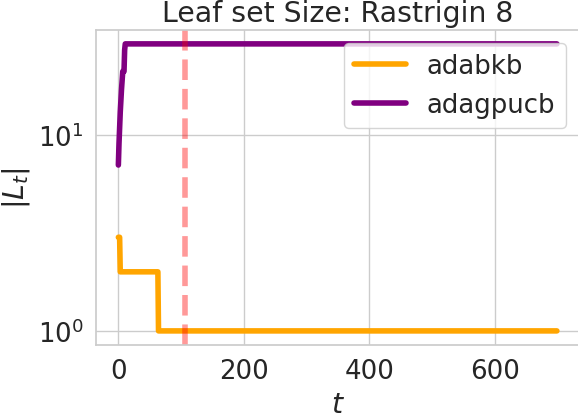}
	\includegraphics[width=0.25\linewidth]{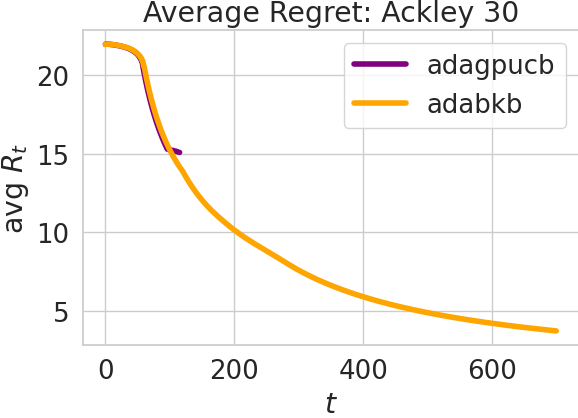}
	\includegraphics[width=0.25\linewidth]{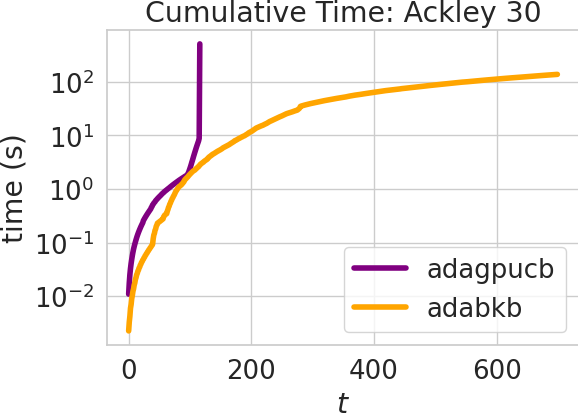}
	\includegraphics[width=0.25\linewidth]{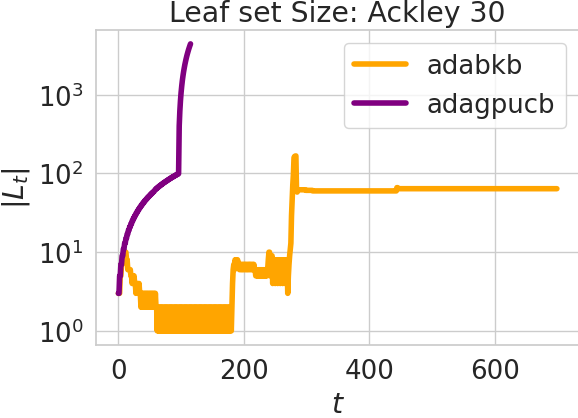}
	\caption{Average regret obtained by the algorithms in optimizing functions in Table \ref{Tab:set}}
	\label{fig:avg_reg_oexp_2}
\end{figure}
As in Section~\ref{sec:experiments}, we plot the average regret, cumulative time and leaf set size per iteration (Figure~\ref{fig:avg_reg_oexp_1} and~\ref{fig:avg_reg_oexp_2}). 
As we expected, (in general) in low dimensional cases GP-UCB is faster than AdaGP-UCB 
because the discretization is composed of few points so the computations are fast and convergency is 
reached in few iterations. In Ada-BKB this problem is faced with the pruning procedure which 
reduces the number of nodes i.e. the number of points in which we have to evaluate the index function. 
In case the number of pruned nodes is $0$ we could expect that in low dimensional cases BKB is faster than 
Ada-BKB (notice in Rosenbrock 2 case that Ada-BKB achieves cumulative time similar to BKB and that the number of the pruned node during iteration is 
lower than the other low dimensional cases). 
However, we can notice that in these cases Ada-BKB is less time expensive than GP-UCB and Ada-GP-UCB. In the worst-case observed, %
it is similar (in time) to BKB.\\
Increasing the dimension of the search space (for instance in Ackley 5), 
Ada-BKB and AdaGP-UCB are faster than GP-UCB and BKB, 
and also the optimum found is better (according to the average regret). In the last line of Figure~\ref{fig:avg_reg_oexp_2}, we couldn't realize the experiments for BKB and GP-UCB because the time cost was too high. Moreover, we can observe that in a $30$-dimensional case, AdaGP-UCB is interrupted due to the time threshold while Ada-BKB is able to complete the $700$ time steps. 
In general, we observe that AdaGP-UCB expands more than Ada-BKB because in AdaGP-UCB there is no 
pruning procedure (and probably because a different expression of $V_h$ is used) which reduce the number of 
nodes allowing to obtain a better performance in time. 
\subsection{Robustness to small pertubation of F}\label{app:small_f}
Since the choice of $F = 1$ is an heuristic, 
we did some synthetic experiments comparing performances of Ada-BKB with different values for $F$. 
\begin{figure}[H]
	\centering
	\vspace{.3in}
	\includegraphics[width=0.23\linewidth]{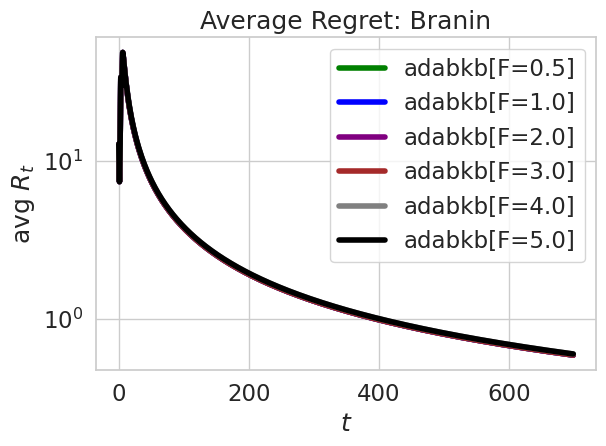}
	\includegraphics[width=0.23\linewidth]{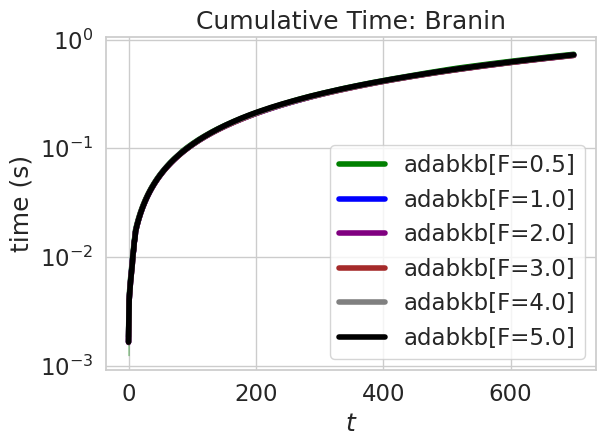}
	\includegraphics[width=0.23\linewidth]{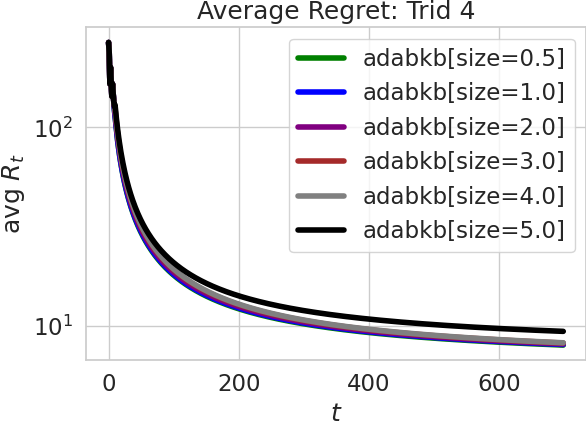}
	\includegraphics[width=0.23\linewidth]{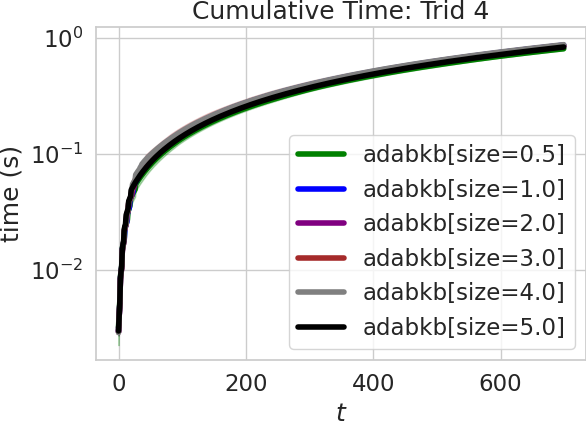}
	\includegraphics[width=0.23\linewidth]{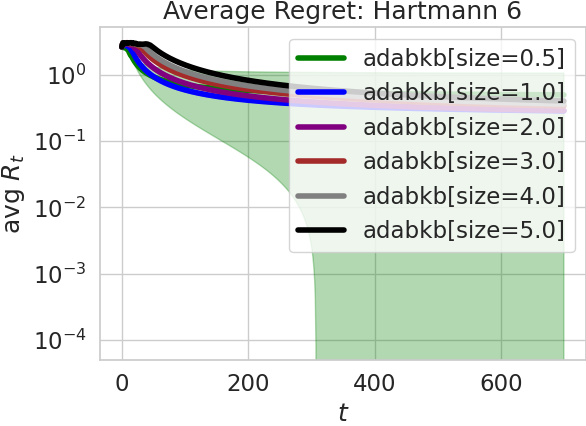}
	\includegraphics[width=0.23\linewidth]{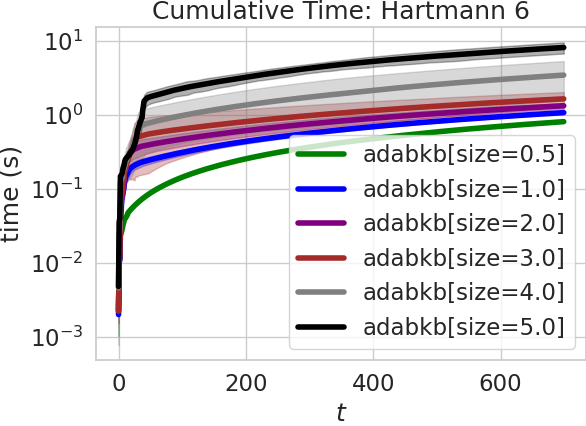}
	\includegraphics[width=0.23\linewidth]{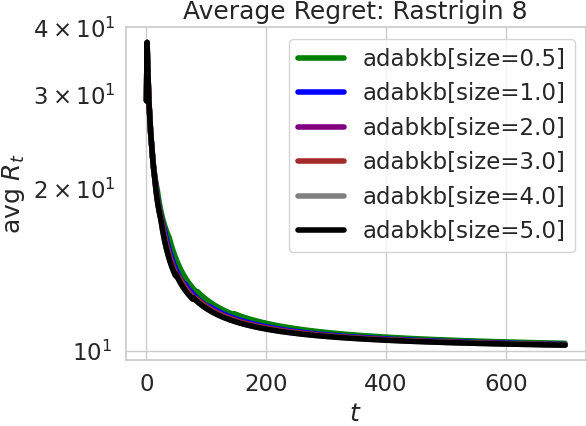}
	\includegraphics[width=0.23\linewidth]{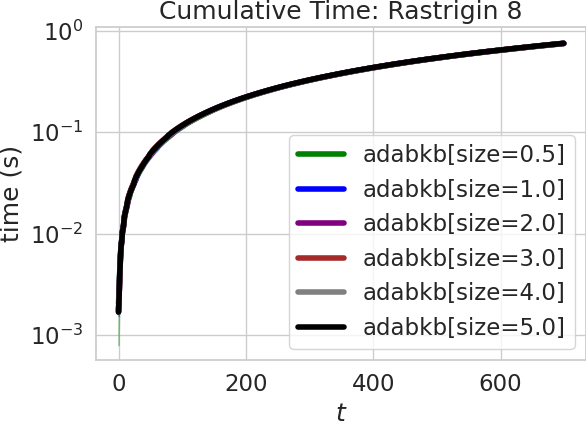}\\
	\vspace{.3in}
	\caption{Average regret and cumulative time of Ada-BKB changing $F$}
\end{figure}
We can observe that for small changes of $F$, results in regret and time are similar 
i.e.~the algorithm is robust to small changes of $F$. 
Obviously, taking $F$ too small will lead to small values for $V_h$ (eq.~\eqref{eqn:vh}) 
and, thus, the algorithm can evaluate centroid more times because of the expansion rule. 
On the other hand, taking $F$ too high can lead to over-expansion.
\begin{figure}[H]
	\centering
	\vspace{.3in}
	\includegraphics[width=0.23\linewidth]{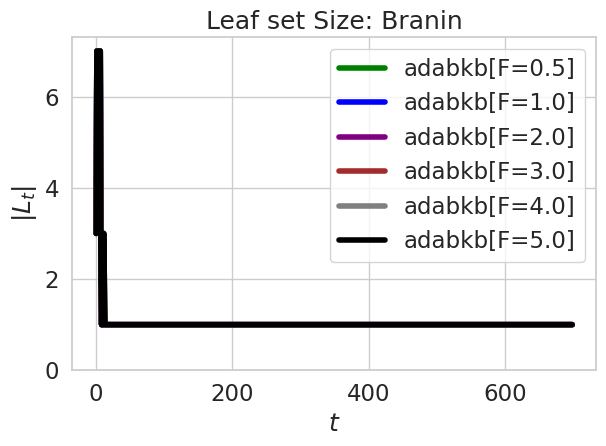}
	\includegraphics[width=0.23\linewidth]{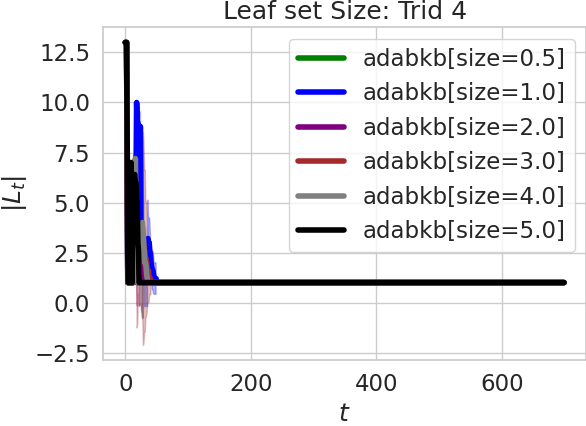}
	\includegraphics[width=0.23\linewidth]{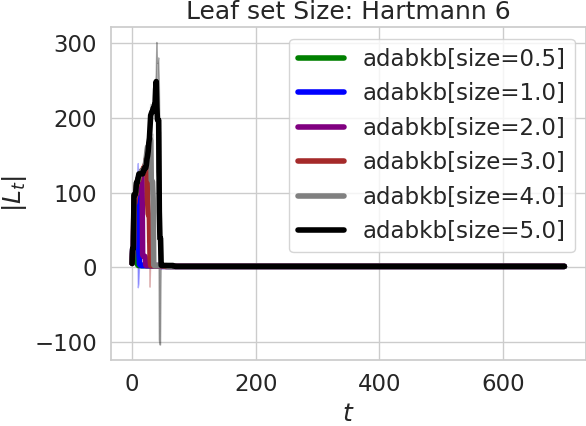}
	\includegraphics[width=0.23\linewidth]{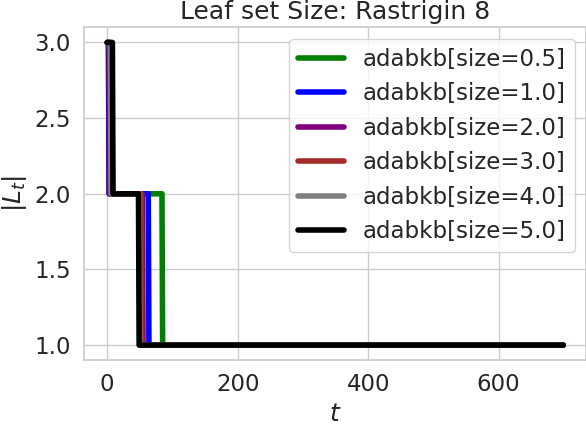}
	\vspace{.3in}
	\caption{Leaf set size per iteration of Ada-BKB changing $F$}  
\end{figure}
\subsection{Partition tree selection}\label{app:choose_N}
In practice, to run Ada-BKB, we have to choose the number of children per node $N$ (see Algorithm~\ref{alg:1}). The choice of a value for this parameter let us choose a partition tree used and explored as indicated in Section~\ref{algo}. Main results (see Section~\ref{theory}) suggest to choice this parameter as small as possible (i.e. $2$ or $3$) since it affects both computational cost and cumulative regret. 
\begin{figure}[H]
	\centering
	\includegraphics[width=0.32\linewidth]{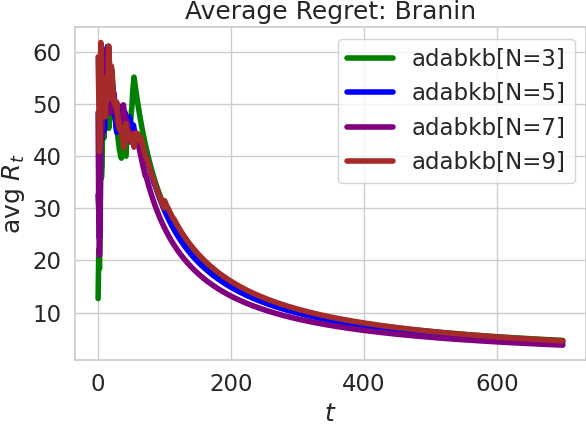}
	\includegraphics[width=0.32\linewidth]{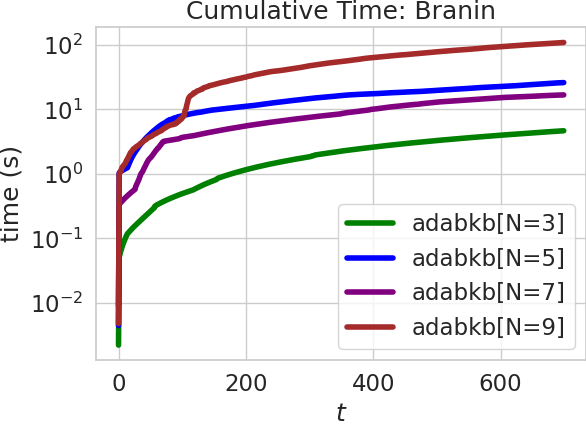}
	\includegraphics[width=0.32\linewidth]{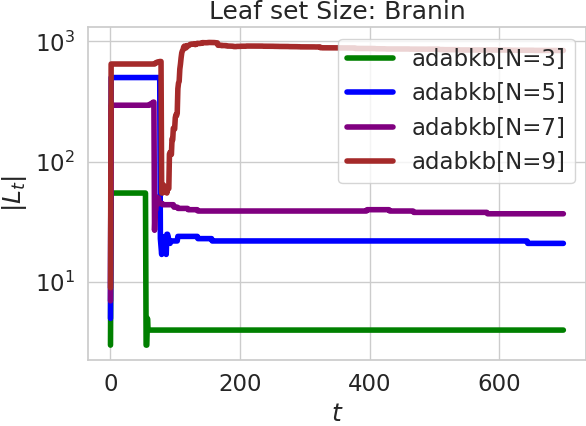}
	\includegraphics[width=0.32\linewidth]{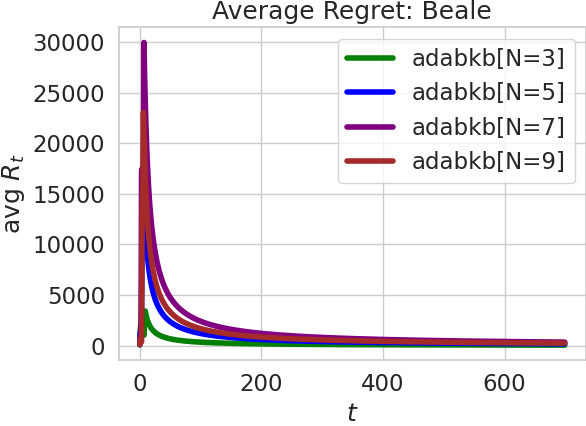}
	\includegraphics[width=0.32\linewidth]{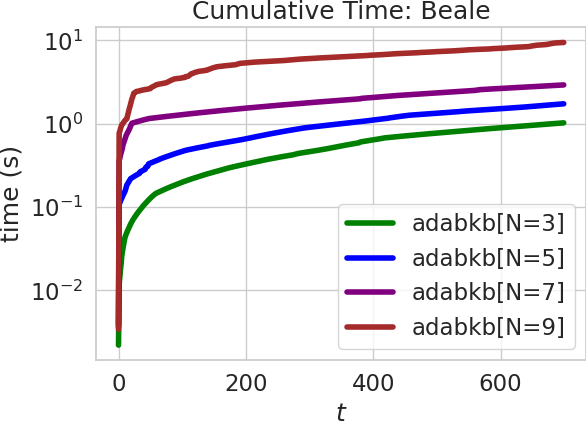}
	\includegraphics[width=0.32\linewidth]{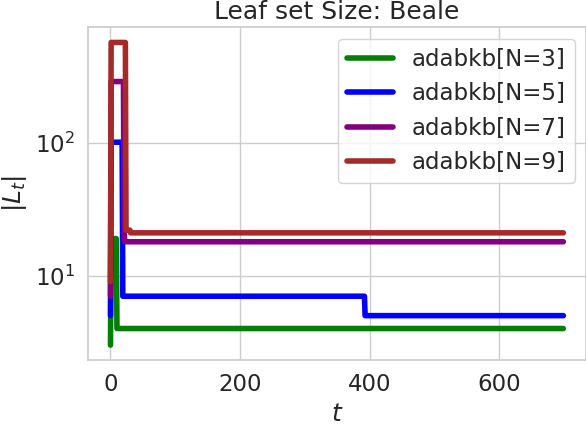}
	\includegraphics[width=0.32\linewidth]{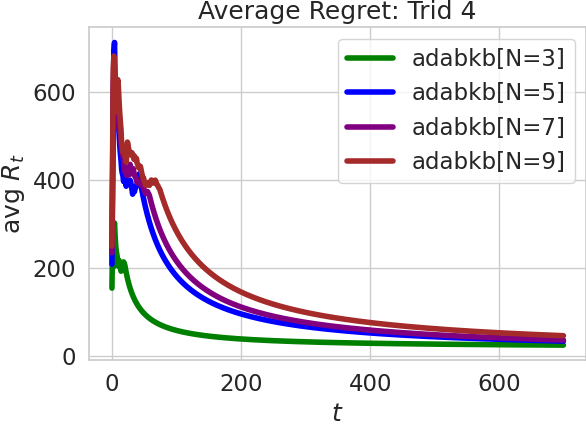}
	\includegraphics[width=0.32\linewidth]{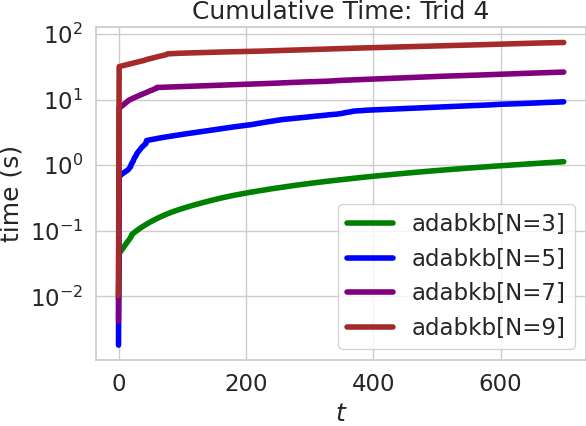}
	\includegraphics[width=0.32\linewidth]{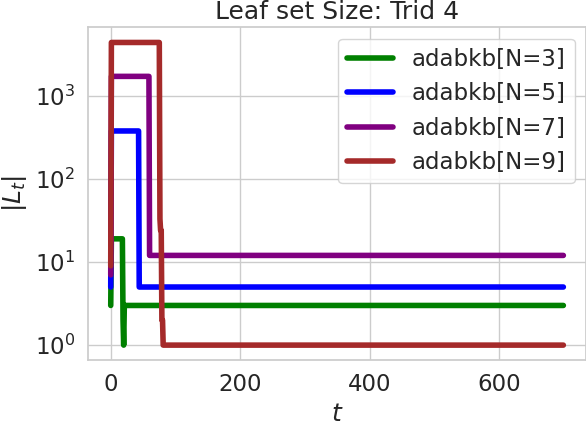}
	\caption{Average regret, cumulative time and leaf set size per iteration of Ada-BKB changing $N$}
	\label{fig:changing_N}  
\end{figure}
Considering a scenario in which we have a depth threshold $h_\text{max}$ and a budget $T$ high enough, we can observe that the number of children per node $N$ doesn't drastically change the best configuration found by the algorithm (see Figure~\ref{fig:changing_N}). 
Obviously, increasing $N$ will require a higher execution time since the cardinality of the leaf set will increase faster. However, in scenarios in which $h_\text{max}$ is low and the search space is a large hypercube, an high number of children per node can be usefull. 
Indeed, an high $N$ allows to produce small partitions faster than small $N$ according to the splitting procedure (see Section~\ref{algo}). 
This let Ada-BKB to provide good performance in regret (in practice) even when the maximum depth threshold $h_\text{max}$ is low.
\begin{figure}[H]
	\centering
	\vspace{.3in}
	\includegraphics[width=0.32\linewidth]{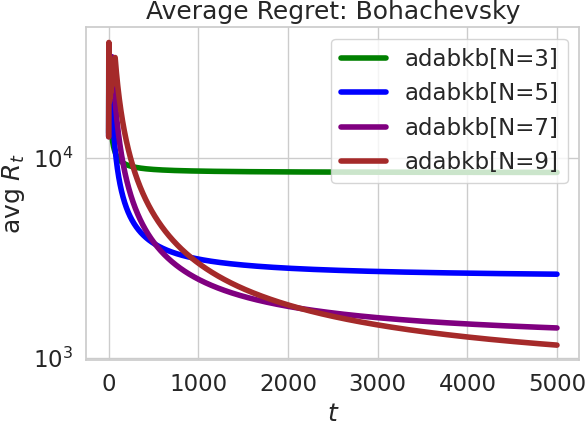}
	\includegraphics[width=0.32\linewidth]{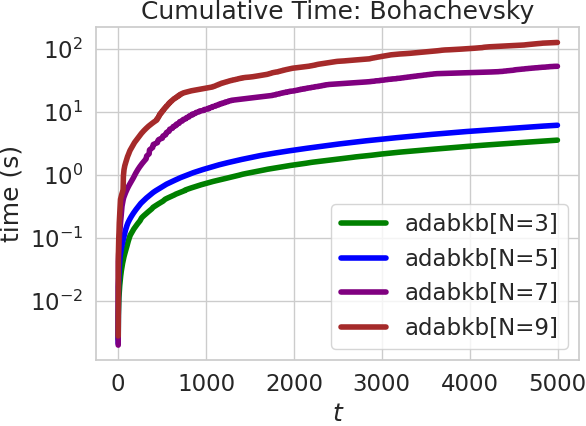}
	\includegraphics[width=0.32\linewidth]{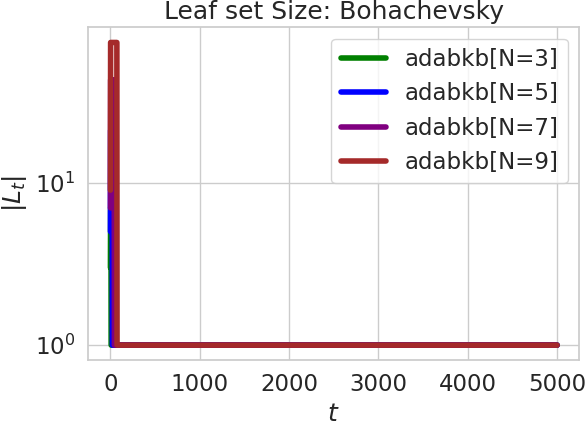}
	\caption{Average regret, cumulative time and leaf set size per iteration of Ada-BKB changing $N$ with $h_\text{max} = 2$}
	\label{fig:changing_N_low_hmax}  
\end{figure}
In Figure~\ref{fig:changing_N_low_hmax}, we optimize Bohachevsky function (see Appendix~\ref{app:experiements} for details on search space) with a maximum depth threshold $h_\text{max} = 2$. In this case, we can observe that increasing $N$, we obtain better results in average regret but it decreases slower as expected (see Theorem~\ref{thm:reg_bounds}). When we performed the experiments, we observed that a good way to select $N$ consists in starting with small values ($2$ or $3$) and increase it if the budget is large enough (which depends from the application), the search space is large and low-dimensional.

\subsection{RandomBKB and Ada-BKB}\label{app:random_bkb_exp}
To show the importance and the strength of adaptive discretizations, we compared Ada-BKB with BKB over a random discretization (called \textit{RandomBKB}). The red vertical dashed line indicates when the early stopping condition is satisfied.
\begin{figure}[H]
	\centering
	\vspace{.3in}
	\includegraphics[width=0.23\linewidth]{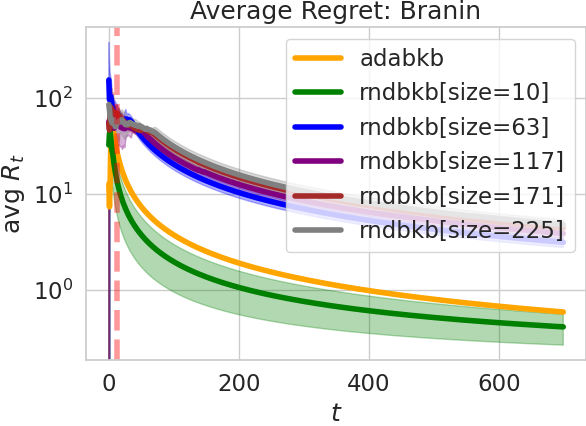}
	\includegraphics[width=0.23\linewidth]{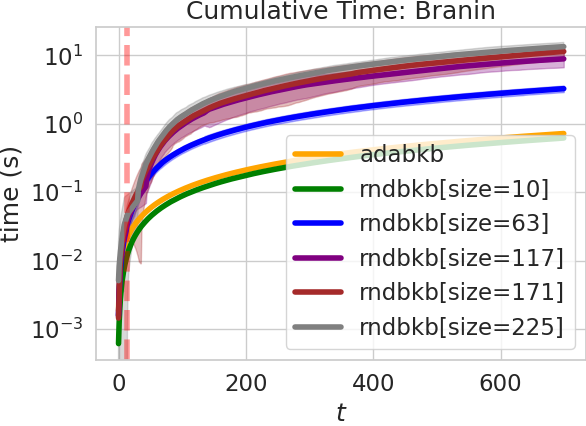}
	\includegraphics[width=0.23\linewidth]{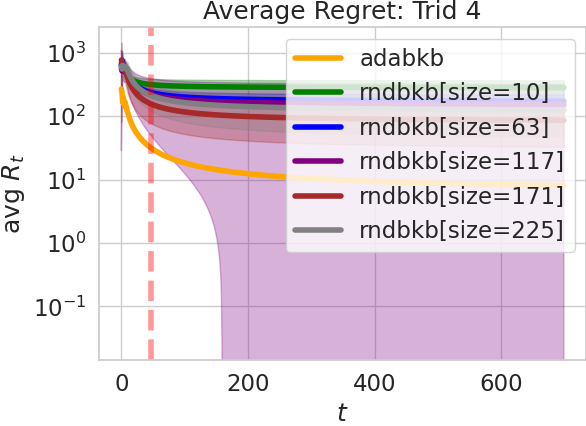}
	\includegraphics[width=0.23\linewidth]{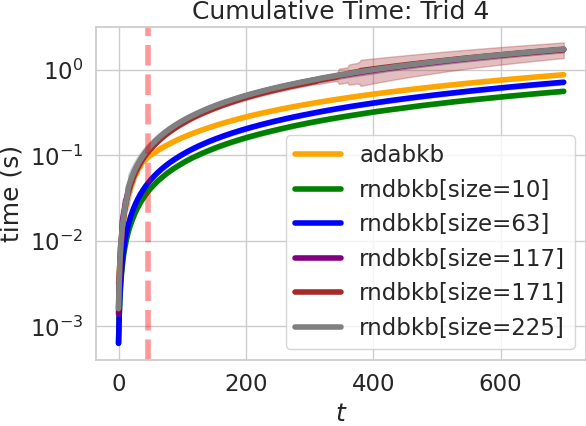}\\
	\includegraphics[width=0.23\linewidth]{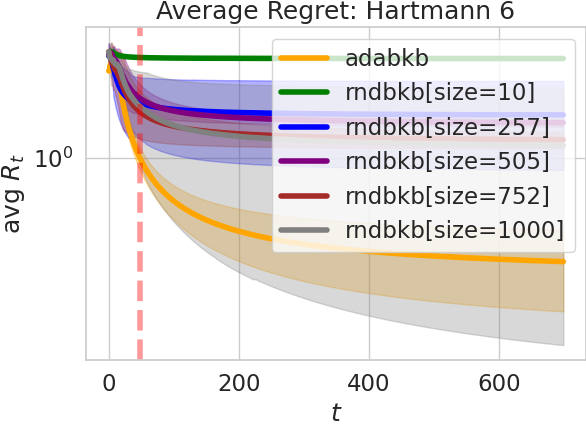}
	\includegraphics[width=0.23\linewidth]{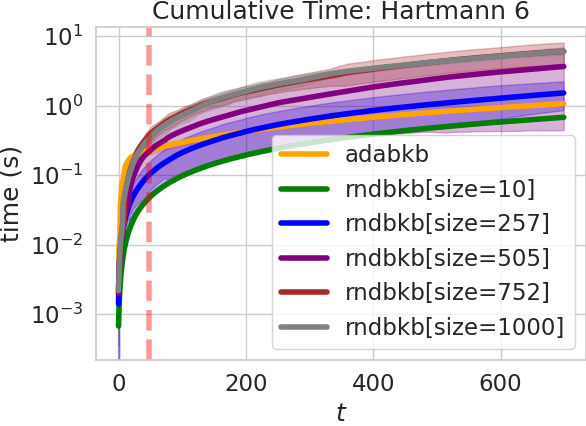}
	\includegraphics[width=0.23\linewidth]{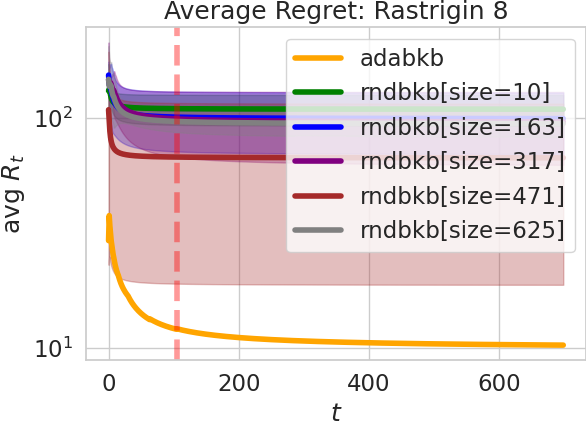}
	\includegraphics[width=0.23\linewidth]{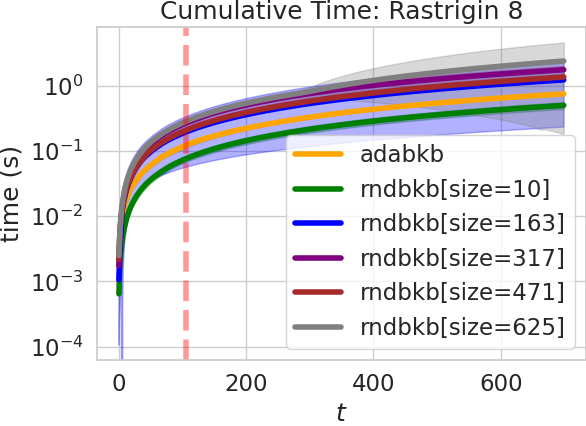}\\
	\vspace{.3in}
	\caption{Average regret and cumulative time of Ada-BKB and RandomBKB with different discretization size}
\end{figure}
As we expected, in low dimensional case (Branin case) it is possible to build a random discretization which contains a sub-optimal configuration. Increasing the dimensions of the search space (as in Rastrigin 8 case), we can observe that even if we increase the size of discretizations used in RandomBKB, 
we still do not obtain results in regret as good as in Ada-BKB. 
This happens because in high dimensional cases the search space is too large and we need to generate many random points to have a good probability of obtaining a search space with suboptimal candidates. However, large discretizations, as we observed in Appendix~\ref{app:other_exp}, will make BKB (and consequently also RandomBKB) very time-expensive due to the computations required to compute the posterior eq.~\eqref{eqn:mu_sig} (indeed, obviously, we can notice that increasing 
the size of the random discretizations, the cumulative time spent to execute RandomBKB increases). 
Moreover, we can notice that Ada-BKB still achieves good performances in time and maintains (in mid and high dimensional search spaces) the best results in regret w.r.t. Random-BKB executions with lower variance (this because RandomBKB does not have a strategy to explore the search space, but it just builds random grids). This shows us that adaptive discretizations are more convenient than random discretizations.
\subsection{Ada-BKB and GP-ThreDS}

Ada-BKB parameters are indicated in Table~\ref{tab:params_gpthreds_adabkb}. The implementation of GP-ThreDS used in these experiments can be downloaded from the official repository: \url{https://github.com/sudeepsalgia/GP_ThreDS}. 
The machine used to performe these experiments is less powerfull than the one described in Appendix~\ref{app:machine}. We decided to use it in order to show that our algorithm can run and provide high performance also in low-powered machines. Details about this machine are reported in Table~\ref{tab:machine}.
We consider the same setting of~\cite{salgia2020computationally} in which the Branin and Rosenbrock functions (defined in the same work) are optimized. As in~\cite{salgia2020computationally}, we will consider a search space $X = [0, 1]^2$ for both functions. The hyperparameters used for GP-ThreDS are indicated in \cite{salgia2020computationally}[Appendix D.1]. The function evaluation budget is set to $T = 700$.
\begin{figure}[H]
	\centering
	\includegraphics[width=0.8\linewidth]{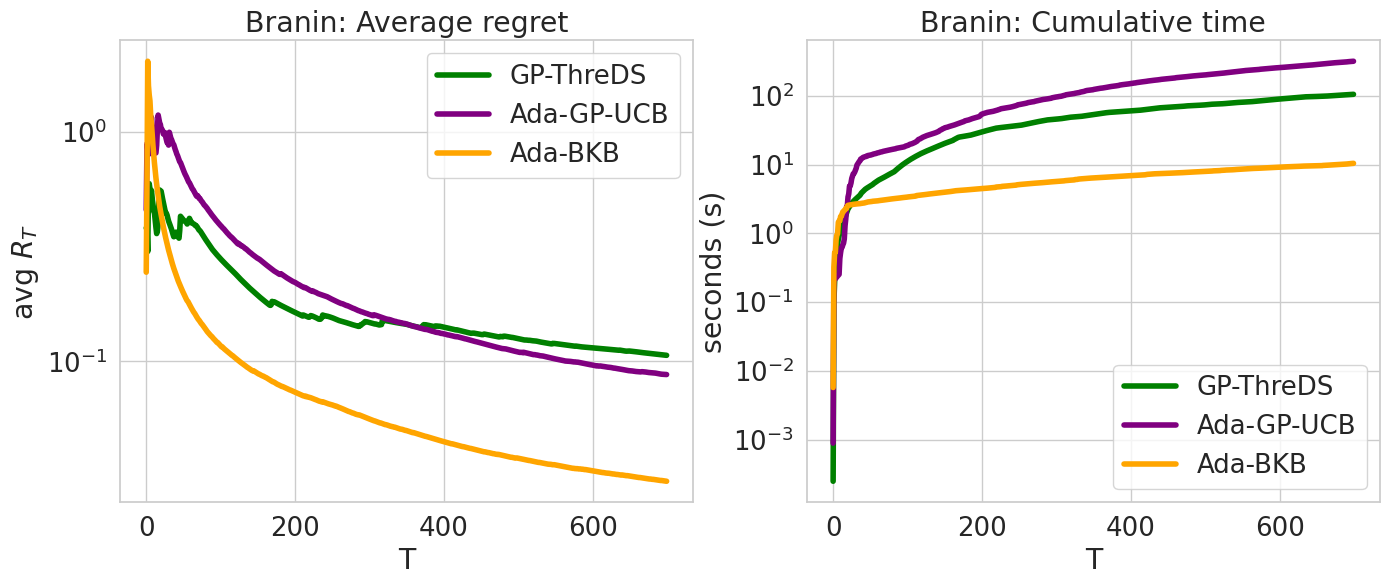}\\
	\includegraphics[width=0.8\linewidth]{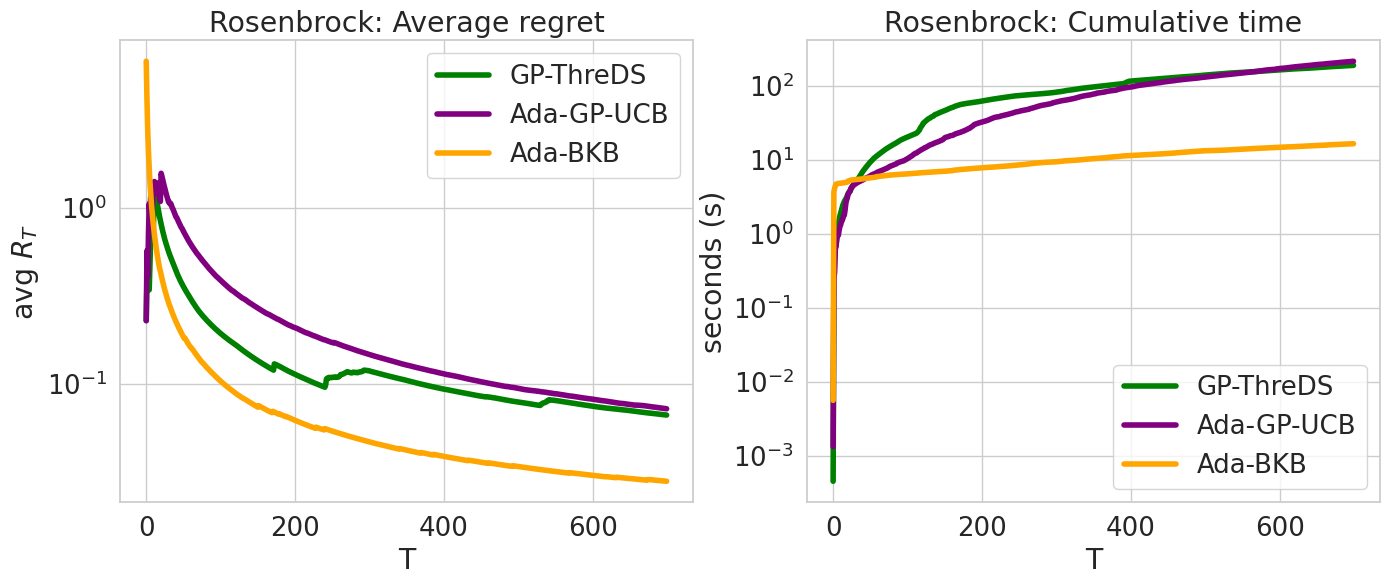}
	\caption{From left to right, average regret and cumulative time of Ada-BKB and GP-ThreDS in optimizing Branin and Rosenbrock functions.}\label{fig:adabkb_gpthreds}
\end{figure}
As we can observe in Figure~\ref{fig:adabkb_gpthreds}, Ada-BKB performs better than GP-ThreDS both in regret and cumulative time. Moreover, we can notice that GP-ThreDS performs better than Ada-GP-UCB in time but performs $\approx 10$ times worse than Ada-BKB (in computational time).
In Table~\ref{tab:time_gpthreds_adabkb}, we report the total time elapsed by three algorithms.
\begin{table}[H]
	\caption{Total time elapsed by algorithms to optimize Branin and Rosenbrock functions} \label{tab:time_gpthreds_adabkb}
	\centering
	\vspace{5px}
	\begin{tabular}{l c c}
		\textbf{ALGORITHM}  &\textbf{BRANIN} &\textbf{ROSENBROCK} \\
		\hline \\
		Ada-GP-UCB & 318.65s & 216.14s\\
		GP-ThreDS  & 105.30s & 190.17s\\
		Ada-BKB    & \textbf{10.43}s & \textbf{16.56}s\\
	\end{tabular}
\end{table}

\begin{table}[H]
	\caption{Parameters of Ada-BKB algorithm to optimize Branin and Rosenbrock functions}  \label{tab:params_gpthreds_adabkb}
	\vspace{5px}
	\centering
	\begin{tabular}{l c c c c c}
		\textbf{FUNCTION}  &\textbf{$\sigma$} &\textbf{$\lambda$} & \textbf{$F$} & \textbf{$N$} & \textbf{$h_\text{max}$}
		\\\hline\\
		Branin & $0.5$ & $0.001$ & $1.0$ & $3$ & $7$\\
		Rosenbrock & $0.5$ & $0.001$ & $1.0$ & $5$ & $5$\\
	\end{tabular}
\end{table}

\begin{table}[H]
	\caption{Machine used to perform these experiments}\label{tab:machine}
	\centering
	\vspace{5px}
	\begin{tabular}{l l}
		\textbf{FEATURE} &\\\hline\\
		OS & Debian 11\\
		CPU & Intel(R) Core(TM) i7-8550U CPU 1.80GHz\\
		RAM & 16 GB
	\end{tabular}
\end{table}

\section{EXPANDED DISCUSSION}\label{app:expanded_discussion}
In this appendix, we discuss the relationship of Algorithm~\ref{alg:1} and the other similar recent algorithms. 
We focus to compare our Ada-BKB with GP-ThreDS~\citep{salgia2020computationally}, 
AdaGP-UCB~\citep{shekhar2018gaussian}, LP-GP-UCB~\citep{shekhar2020multi} and BKB~\citep{calandriello2019gaussian}. 
Despite BKB, our algorithm can work on continuous search spaces without building an offline discretization which can be very expensive, see Appendix~\ref{app:other_exp}(notice that using random discretizations doesn't provide good results in high-dimensional search spaces, see Appendix~\ref{app:random_bkb_exp}). 
We followed the direction indicated in~\citep{shekhar2020multi} to sketch the model confirming and proving
that we get better performance in time. We also noticed that using a partition schema as in~\citep{shekhar2018gaussian}, let us obtain similar or potentially improved regret bounds with a lower computational cost:
\begin{equation*}
	\begin{aligned}
		&(\text{LP-GP-UCB Regret:}) \qquad  \mathcal{O}(\sqrt{T}d_\text{eff}(T))\\
		&(\text{Ada-BKB Regret:}) \qquad \mathcal{O} (\sqrt{T}d_\text{eff}(T)\log T) \qquad \text{or} \qquad \mathcal{O} \Bigg(\sqrt{Td_\text{eff}(T)\log T \frac{N^{h_\text{max}} - 1}{N - 1}} \Bigg)
	\end{aligned}
\end{equation*}
Moreover, introducing a pruning procedure and an early stopping condition, we observed in the experiments (see Appendix~\ref{app:other_exp}) that we can further 
reduce the time-cost in practice. 
\paragraph*{BKB and SVGP.} 
This work open other directions in particular in using different sketching models as SVGP~\citep{pmlr-v5-titsias09a, pmlr-v97-burt19a} which mainly differs from BKB 
for inducing point selection. While in SVGP, inducing points are selected by maximizing the \textit{evidence lower bound} (ELBO)~\citep{pmlr-v38-hensman15}, BKB uses a procedure called 
\textit{resparsification} which provides guarantees on the size of the set containing the inducing points~\cite[Theorem 1]{calandriello2019gaussian}. Moreover, as shown in~\citep{shekhar2018gaussian}, using a Gaussian Process let us avoid to include in the $V_h$ expression (eq.~\eqref{eqn:vh}) the norm of the reward function $f$ which is not 
known a priori. In our experiments, we observed that a valid heuristic consists in setting it as $1$ (see also Appendix~\ref{app:small_f}). 
\paragraph*{Tuning the hyper-parameters of the model}
As shown in~\citep{nystrom_svgp,calandriello2019gaussian}, BKB is equivalent to a 
DTC approximation of a Gaussian Process~\citep{sparse_gp_approx} and thus, in practice, 
we can tune the hyper-parameters of BKB by maximizing the marginal likelihood.

\end{document}

%% file: MacroBook.tex
\newcommand{\Nystrom}{\text{Nystr\"om}}
\newcommand{\ti}[1]{\widetilde{#1}}

\def\argmax{\operatornamewithlimits{arg\,max}}

\newtheorem*{theorem*}{Theorem}

\providecommand{\nor}[1]{\left\lVert {#1} \right\rVert}

\providecommand{\scal}[2]{\left\langle{#1},{#2}\right\rangle}

\newcommand{\R}{\mathbb R}

\newcommand{\N}{\mathbb N}

\newcommand{\hh}{\mathcal H}

\newcommand{\la}{\lambda}

\newcommand{\X}{X}

\newcommand{\eps}{\varepsilon}

%% file: AdaBKB_1251.bbl
\begin{thebibliography}{}

\bibitem[Bubeck et~al., 2011]{bubeck2011x}
Bubeck, S., Munos, R., Stoltz, G., and Szepesv{\'a}ri, C. (2011).
\newblock X-armed bandits.
\newblock {\em Journal of Machine Learning Research}, 12:1655--1695.

\bibitem[Buitinck et~al., 2013]{sklearn_api}
Buitinck, L., Louppe, G., Blondel, M., Pedregosa, F., Mueller, A., Grisel, O.,
  Niculae, V., Prettenhofer, P., Gramfort, A., Grobler, J., Layton, R.,
  VanderPlas, J., Joly, A., Holt, B., and Varoquaux, G. (2013).
\newblock {API} design for machine learning software: experiences from the
  scikit-learn project.
\newblock In {\em ECML PKDD Workshop: Languages for Data Mining and Machine
  Learning}, pages 108--122.

\bibitem[Burt et~al., 2019]{pmlr-v97-burt19a}
Burt, D., Rasmussen, C.~E., and Van Der~Wilk, M. (2019).
\newblock Rates of convergence for sparse variational {G}aussian process
  regression.
\newblock In Chaudhuri, K. and Salakhutdinov, R., editors, {\em Proceedings of
  the 36th International Conference on Machine Learning}, volume~97 of {\em
  Proceedings of Machine Learning Research}, pages 862--871. PMLR.

\bibitem[Calandriello et~al., 2019]{calandriello2019gaussian}
Calandriello, D., Carratino, L., Lazaric, A., Valko, M., and Rosasco, L.
  (2019).
\newblock Gaussian process optimization with adaptive sketching: Scalable and
  no regret.
\newblock In {\em Conference on Learning Theory}, pages 533--557. PMLR.

\bibitem[Calandriello et~al., 2020]{calandriello2020near}
Calandriello, D., Carratino, L., Lazaric, A., Valko, M., and Rosasco, L.
  (2020).
\newblock Near-linear time gaussian process optimization with adaptive batching
  and resparsification.
\newblock In {\em International Conference on Machine Learning}, pages
  1295--1305. PMLR.

\bibitem[Calandriello et~al., 2022]{calandriello2022scaling}
Calandriello, D., Carratino, L., Lazaric, A., Valko, M., and Rosasco, L.
  (2022).
\newblock Scaling gaussian process optimization by evaluating a few unique
  candidates multiple times.
\newblock {\em arXiv preprint arXiv:2201.12909}.

\bibitem[Drineas et~al., 2005]{drineas2005nystrom}
Drineas, P., Mahoney, M.~W., and Cristianini, N. (2005).
\newblock On the {N}ystr{\"o}m method for approximating a {G}ram matrix for
  improved kernel-based learning.
\newblock {\em Journal of Machine Learning Research}, 6(12):2153--2175.

\bibitem[Dua and Graff, 2017]{Dua:2019}
Dua, D. and Graff, C. (2017).
\newblock {UCI} machine learning repository.

\bibitem[Gardner et~al., 2018]{gardner2018gpytorch}
Gardner, J., Pleiss, G., Weinberger, K.~Q., Bindel, D., and Wilson, A.~G.
  (2018).
\newblock Gpytorch: Blackbox matrix-matrix gaussian process inference with gpu
  acceleration.
\newblock In Bengio, S., Wallach, H., Larochelle, H., Grauman, K.,
  Cesa-Bianchi, N., and Garnett, R., editors, {\em Advances in Neural
  Information Processing Systems}, volume~31. Curran Associates, Inc.

\bibitem[Harris et~al., 2020]{harris2020array}
Harris, C.~R., Millman, K.~J., van~der Walt, S.~J., Gommers, R., Virtanen, P.,
  Cournapeau, D., Wieser, E., Taylor, J., Berg, S., Smith, N.~J., Kern, R.,
  Picus, M., Hoyer, S., van Kerkwijk, M.~H., Brett, M., Haldane, A., del
  R{\'{i}}o, J.~F., Wiebe, M., Peterson, P., G{\'{e}}rard-Marchant, P.,
  Sheppard, K., Reddy, T., Weckesser, W., Abbasi, H., Gohlke, C., and Oliphant,
  T.~E. (2020).
\newblock Array programming with {NumPy}.
\newblock {\em Nature}, 585(7825):357--362.

\bibitem[Hensman et~al., 2015]{pmlr-v38-hensman15}
Hensman, J., Matthews, A., and Ghahramani, Z. (2015).
\newblock {Scalable Variational Gaussian Process Classification}.
\newblock In Lebanon, G. and Vishwanathan, S. V.~N., editors, {\em Proceedings
  of the Eighteenth International Conference on Artificial Intelligence and
  Statistics}, volume~38 of {\em Proceedings of Machine Learning Research},
  pages 351--360, San Diego, California, USA. PMLR.

\bibitem[Kleinberg et~al., 2008]{kleinberg2008multi}
Kleinberg, R., Slivkins, A., and Upfal, E. (2008).
\newblock Multi-armed bandits in metric spaces.
\newblock In {\em Proceedings of the fortieth annual ACM symposium on Theory of
  computing}, pages 681--690.

\bibitem[Kleinberg et~al., 2013]{kleinberg2013bandits}
Kleinberg, R., Slivkins, A., and Upfal, E. (2013).
\newblock Bandits and experts in metric spaces.
\newblock {\em arXiv preprint arXiv:1312.1277}.

\bibitem[Kung, 2014]{kung_2014}
Kung, S.~Y. (2014).
\newblock {\em Kernel Methods and Machine Learning}.
\newblock Cambridge University Press.

\bibitem[Lattimore and Szepesv{\'a}ri, 2020]{lattimore2020bandit}
Lattimore, T. and Szepesv{\'a}ri, C. (2020).
\newblock {\em Bandit algorithms}.
\newblock Cambridge University Press.

\bibitem[Lyon et~al., 2016]{Lyon_2016}
Lyon, R.~J., Stappers, B.~W., Cooper, S., Brooke, J.~M., and Knowles, J.~D.
  (2016).
\newblock Fifty years of pulsar candidate selection: from simple filters to a
  new principled real-time classification approach.
\newblock {\em Monthly Notices of the Royal Astronomical Society},
  459(1):1104–1123.

\bibitem[Meanti et~al., 2020]{falkonlibrary2020}
Meanti, G., Carratino, L., Rosasco, L., and Rudi, A. (2020).
\newblock Kernel methods through the roof: Handling billions of points
  efficiently.
\newblock In Larochelle, H., Ranzato, M., Hadsell, R., Balcan, M.~F., and Lin,
  H., editors, {\em Advances in Neural Information Processing Systems},
  volume~33, pages 14410--14422. Curran Associates, Inc.

\bibitem[Munos, 2011]{munos2011optimistic}
Munos, R. (2011).
\newblock Optimistic optimization of a deterministic function without the
  knowledge of its smoothness.
\newblock In Shawe-Taylor, J., Zemel, R., Bartlett, P., Pereira, F., and
  Weinberger, K.~Q., editors, {\em Advances in Neural Information Processing
  Systems}, volume~24. Curran Associates, Inc.

\bibitem[Munos, 2014]{munos2014bandits}
Munos, R. (2014).
\newblock From bandits to monte-carlo tree search: The optimistic principle
  applied to optimization and planning.
\newblock {\em Foundations and Trends in Machine Learning}, 7(1):1--129.

\bibitem[Mutn{\`y} and Krause, 2019]{mutny2019efficient}
Mutn{\`y}, M. and Krause, A. (2019).
\newblock Efficient high dimensional bayesian optimization with additivity and
  quadrature fourier features.
\newblock {\em Advances in Neural Information Processing Systems 31}, pages
  9005--9016.

\bibitem[Nesterov, 2014]{nesterov}
Nesterov, Y. (2014).
\newblock {\em Introductory Lectures on Convex Optimization: A Basic Course}.
\newblock Springer Publishing Company, Incorporated, 1 edition.

\bibitem[Nesterov and Spokoiny, 2017]{nesterov2017random}
Nesterov, Y. and Spokoiny, V. (2017).
\newblock Random gradient-free minimization of convex functions.
\newblock {\em Foundations of Computational Mathematics}, 17(2):527--566.

\bibitem[Paszke et~al., 2017]{paszke2017automatic}
Paszke, A., Gross, S., Chintala, S., Chanan, G., Yang, E., DeVito, Z., Lin, Z.,
  Desmaison, A., Antiga, L., and Lerer, A. (2017).
\newblock Automatic differentiation in pytorch.

\bibitem[Pedregosa et~al., 2011]{scikit-learn}
Pedregosa, F., Varoquaux, G., Gramfort, A., Michel, V., Thirion, B., Grisel,
  O., Blondel, M., Prettenhofer, P., Weiss, R., Dubourg, V., Vanderplas, J.,
  Passos, A., Cournapeau, D., Brucher, M., Perrot, M., and Duchesnay, E.
  (2011).
\newblock Scikit-learn: Machine learning in {P}ython.
\newblock {\em Journal of Machine Learning Research}, 12:2825--2830.

\bibitem[Qin et~al., 2017]{qin2017improving}
Qin, C., Klabjan, D., and Russo, D. (2017).
\newblock Improving the expected improvement algorithm.
\newblock In Guyon, I., Luxburg, U.~V., Bengio, S., Wallach, H., Fergus, R.,
  Vishwanathan, S., and Garnett, R., editors, {\em Advances in Neural
  Information Processing Systems}, volume~30. Curran Associates, Inc.

\bibitem[Qui\~{n}onero Candela and Rasmussen, 2005]{sparse_gp_approx}
Qui\~{n}onero Candela, J. and Rasmussen, C.~E. (2005).
\newblock A unifying view of sparse approximate gaussian process regression.
\newblock {\em J. Mach. Learn. Res.}, 6:1939–1959.

\bibitem[Rasmussen, 2003]{rasmussen2003gaussian}
Rasmussen, C.~E. (2003).
\newblock Gaussian processes in machine learning.
\newblock In {\em Summer school on machine learning}, pages 63--71. Springer.

\bibitem[Rudi et~al., 2017]{rudi2017falkon}
Rudi, A., Carratino, L., and Rosasco, L. (2017).
\newblock Falkon: An optimal large scale kernel method.
\newblock In Guyon, I., Luxburg, U.~V., Bengio, S., Wallach, H., Fergus, R.,
  Vishwanathan, S., and Garnett, R., editors, {\em Advances in Neural
  Information Processing Systems}, volume~30. Curran Associates, Inc.

\bibitem[Salgia et~al., 2020]{salgia2020computationally}
Salgia, S., Vakili, S., and Zhao, Q. (2020).
\newblock A computationally efficient approach to black-box optimization using
  gaussian process models.
\newblock {\em arXiv preprint arXiv:2010.13997}.

\bibitem[Shekhar and Javidi, 2018]{shekhar2018gaussian}
Shekhar, S. and Javidi, T. (2018).
\newblock {{G}aussian process bandits with adaptive discretization}.
\newblock {\em Electronic Journal of Statistics}, 12(2):3829 -- 3874.

\bibitem[Shekhar and Javidi, 2020]{shekhar2020multi}
Shekhar, S. and Javidi, T. (2020).
\newblock Multi-scale zero-order optimization of smooth functions in an rkhs.

\bibitem[Srinivas et~al., 2012]{srinivas2012gaussian}
Srinivas, N., Krause, A., Kakade, S., and Seeger, M. (2012).
\newblock Information-theoretic regret bounds for gaussian process optimization
  in the bandit setting.
\newblock {\em IEEE Transactions on Information Theory - TIT}, 58:3250--3265.

\bibitem[Srinivas et~al., 2010]{srinivas2009gaussian}
Srinivas, N., Krause, A., Kakade, S.~M., and Seeger, M. (2010).
\newblock {G}aussian process optimization in the bandit setting: No regret and
  experimental design.
\newblock In {\em Proceedings of the 27th International Conference on
  International Conference on Machine Learning}, pages 1015--1022.

\bibitem[Titsias, 2009]{pmlr-v5-titsias09a}
Titsias, M. (2009).
\newblock Variational learning of inducing variables in sparse gaussian
  processes.
\newblock In van Dyk, D. and Welling, M., editors, {\em Proceedings of the
  Twelth International Conference on Artificial Intelligence and Statistics},
  volume~5 of {\em Proceedings of Machine Learning Research}, pages 567--574,
  Hilton Clearwater Beach Resort, Clearwater Beach, Florida USA. PMLR.

\bibitem[Valko et~al., 2013a]{valko2013stochastic}
Valko, M., Carpentier, A., and Munos, R. (2013a).
\newblock Stochastic simultaneous optimistic optimization.
\newblock In {\em International Conference on Machine Learning}, pages 19--27.
  PMLR.

\bibitem[Valko et~al., 2013b]{kernelUCB}
Valko, M., Korda, N., Munos, R., Flaounas, I., and Cristianini, N. (2013b).
\newblock Finite-time analysis of kernelised contextual bandits.
\newblock In {\em Proceedings of the Twenty-Ninth Conference on Uncertainty in
  Artificial Intelligence}, page 654?666.

\bibitem[Wang et~al., 2014]{wang2014bayesian}
Wang, Z., Shakibi, B., Jin, L., and Freitas, N. (2014).
\newblock Bayesian multi-scale optimistic optimization.
\newblock In {\em Artificial Intelligence and Statistics}, pages 1005--1014.
  PMLR.

\bibitem[Whitley, 1994]{whitley1994genetic}
Whitley, D. (1994).
\newblock A genetic algorithm tutorial.
\newblock {\em Statistics and computing}, 4(2):65--85.

\bibitem[{Wild} et~al., 2021]{nystrom_svgp}
{Wild}, V., {Kanagawa}, M., and {Sejdinovic}, D. (2021).
\newblock {Connections and Equivalences between the Nystr{\"o}m Method and
  Sparse Variational Gaussian Processes}.
\newblock {\em arXiv e-prints}, page arXiv:2106.01121.

\end{thebibliography}
